\documentclass[11pt]{article}

\usepackage{times}

\usepackage{enumitem}
\usepackage{color}
\usepackage{graphicx}
\usepackage{booktabs}
\usepackage{amsmath}
\usepackage{amsfonts}
\usepackage{algorithm}
\usepackage{algorithmic}
\usepackage{wrapfig}
\usepackage{amsmath,amsthm,amsfonts,amssymb}
\usepackage{fullpage}
\usepackage[pdftex,bookmarksnumbered,bookmarksopen,
colorlinks,citecolor=blue,linkcolor=blue]{hyperref}

\newcommand{\cC}{\mathcal{C}}
\newcommand{\cD}{\mathcal{D}}
\newcommand{\cP}{\mathcal{P}}
\newcommand{\E}{\mathbb{E}}

\newcommand{\R}{\mathbb{R}}
\newcommand{\err}{\textup{\textsf{err}}}
\newcommand{\red}[1]{{\color{black}#1}}
\newcommand{\sign}{\textup{sign}}
\newcommand{\comment}[1]{}
\newcommand{\ball}{\textup{\textsf{ball}}}
\newenvironment{proofoutline}{\noindent{\textbf{Proof Sketch }}}{\hfill$\blacksquare$\medskip}

\newtheorem{theorem}{Theorem}
\newtheorem{definition}{Definition}
\newtheorem{lemma}[theorem]{Lemma}

\title{Sample and Computationally Efficient Learning Algorithms under S-Concave Distributions}
\usepackage{times}

\author{Maria-Florina Balcan\thanks{Carnegie Mellon University. Email: ninamf@cs.cmu.edu} \and
Hongyang Zhang\thanks{Corresponding author. Carnegie Mellon University. Email: hongyanz@cs.cmu.edu}
}

\date{}

\begin{document}

\maketitle

\begin{abstract}
We provide new results for noise-tolerant and sample-efficient learning algorithms under $s$-concave distributions. The new class of $s$-concave distributions is a broad and natural generalization of log-concavity, and includes many important additional distributions, e.g., the Pareto distribution and $t$-distribution. This class has been studied in the context of efficient sampling, integration, and optimization, but  much remains unknown about the geometry of this class of distributions and their applications in the context of learning. The challenge is that unlike the commonly used distributions in learning (uniform or more generally log-concave distributions), this broader class is not closed under the marginalization operator and many such distributions are fat-tailed. In this work, we introduce new convex geometry tools to study the properties of $s$-concave distributions and use these properties to provide bounds on quantities of interest to learning including the probability of disagreement between two halfspaces, disagreement outside a band, and the disagreement coefficient. We use these results to significantly generalize prior results for margin-based active learning, disagreement-based active learning, and passive learning of intersections of halfspaces. Our analysis of geometric properties of $s$-concave distributions might be of independent interest to optimization more broadly.
\end{abstract}

\section{Introduction}
Developing provable learning algorithms is one of the central challenges in learning theory. The study of such algorithms has led to significant advances in both the theory and practice of passive and active learning. In the passive learning model, the learning algorithm has access to a set of labeled examples sampled i.i.d. from some unknown distribution over the instance space and labeled according to some underlying target function. In the active learning model, however, the algorithm can access unlabeled examples and request labels of its own choice, and the goal is to learn the target function with significantly fewer labels. In this work, we study both learning models in the case where the underlying distribution belongs to the class of $s$-concave distributions.

Prior work on noise-tolerant and sample-efficient algorithms mostly relies on the assumption that the distribution over the instance space is log-concave~\cite{applegate1991sampling,caramanis2007inequality,balcan2013active,xu2017noise}. A distribution is \emph{log-concave} if the logarithm of its density is a concave function. The assumption of log-concavity has been made for a few purposes: for computational efficiency reasons and for sample efficiency reasons. For computational efficiency reasons, it was made to obtain a noise-tolerant algorithm even for seemingly simple decision surfaces like linear separators. These simple algorithms exist for noiseless scenarios, e.g., via linear programming~\cite{servedio2001efficient}, but they are notoriously hard once we have noise~\cite{daniely2016complexity,klivans2014embedding,guruswami2009hardness}; This is why progress on noise-tolerant algorithms has focused on uniform~\cite{kalai2008agnostically,klivans2009learning} and log-concave distributions~\cite{awasthi2017power}. Other concept spaces, like intersections of halfspaces, even have no computationally efficient algorithm in the noise-free settings that works under general distributions, but there has been nice progress under uniform and log-concave distributions~\cite{klivans2009baum}. For sample efficiency reasons,  in the context of active learning,  we need distributional assumptions in order to obtain label complexity improvements~\cite{dasgupta2004analysis}. The most concrete and general class for which prior work obtains such improvements is when the marginal distribution over instance space satisfies log-concavity~\cite{zhang2014beyond,balcan2013active}.
In this work, we provide a broad generalization of all above results, showing how they extend to $s$-concave distributions $(s<0)$. A distribution with density $f(x)$ is \emph{$s$-concave} if $f(x)^s$ is a concave function. We identify key properties of these distributions that allow us to simultaneously extend all above results.

\medskip
\noindent{\textbf{How general and important is the class of s-concave distributions?}}
The class of $s$-concave distributions is very broad and contains many well-known (classes of) distributions as special cases. For example, when $s\rightarrow 0$, $s$-concave distributions reduce to \emph{log-concave} distributions. Furthermore, the $s$-concave class contains infinitely many fat-tailed distributions that do not belong to the class of log-concave distributions, e.g., Cauchy, Pareto, and $t$ distributions, which have been widely applied in the context of theoretical physics and economics, but much remains unknown about how the provable learning algorithms, such as active learning of halfspaces, perform under these realistic distributions.
We also compare $s$-concave distributions with nearly-log-concave distributions, a slightly broader class of distributions than log-concavity. A distribution with density $f(x)$ is nearly-log-concave if for any $\lambda\in[0,1]$, $x_1,x_2\in\R^n$, we have $f(\lambda x_1+(1-\lambda)x_2)\ge e^{-0.0154}f(x_1)^\lambda f(x_2)^{1-\lambda}$~\cite{balcan2013active}. The class of $s$-concave distributions includes many important extra distributions which do not belong to the nearly-log-concave distributions: a nearly-log-concave distribution must have sub-exponential tails~(see Theorem 11, \cite{balcan2013active}), while the tail probability of an $s$-concave distribution might decay much slower (see Theorem \ref{theorem: geometry of s-concavity} (6)).
We also note that efficient sampling, integration and optimization algorithms for $s$-concave distributions have been well understood~\cite{chandrasekaran2009sampling,kalai2006simulated}. Our analysis of $s$-concave distributions bridges these algorithms to the strong guarantees of noise-tolerant and sample-efficient learning algorithms.

\subsection{Our Contributions}

\medskip
\noindent{\textbf{Structural Results.}}
We study various geometric properties of $s$-concave distributions. These properties serve as the structural results for many provable learning algorithms, e.g., margin-based active learning~\cite{balcan2013active}, disagreement-based active learning~\cite{wang2011smoothness,hanneke2014theory}, learning intersections of halfspaces~\cite{klivans2009baum}, etc. When $s\rightarrow 0$, our results exactly reduce to those for log-concave distributions~\cite{balcan2013active,awasthi2016learning,awasthi2017power}. Below, we state our structural results informally:

\begin{theorem}[Informal]
\label{theorem: geometry of s-concavity}
Let $\cD$ be an isotropic $s$-concave distribution in $\R^n$. Then there exist \textbf{closed-form} functions $\gamma(s,m)$, $f_1(s, n)$, $f_2(s, n)$, $f_3(s, n)$, $f_4(s, n)$, and $f_5(s, n)$ such that
\begin{enumerate}[leftmargin=*]\setlength{\itemsep}{-0.1cm}
\item
\emph{(Weakly Closed under Marginal)} The marginal of $\cD$ over $m$ arguments (or cumulative distribution function, CDF) is isotropic $\gamma(s,m)$-concave. (Theorems \ref{theorem: marginal}, \ref{theorem: distribution function})
\item
\emph{(Lower Bound on Hyperplane Disagreement)} For any two unit vectors $u$ and $v$ in $\R^n$, $f_1(s, n)\theta(u,v)\le\Pr_{x\sim \cD}[\sign(u\cdot x)\not=\sign(v\cdot x)]$, where $\theta(u,v)$ is the angle between $u$ and $v$. (Theorem \ref{theorem: disagreement and angle})
\item
\emph{(Probability of Band)} There is a function $d(s, n)$ such that for any unit vector $w\in\R^n$ and any $0<t\le d(s,n)$, we have $f_2(s, n)t< \Pr_{x\sim\cD}[|w\cdot x|\le t]\le f_3(s, n)t$. (Theorem \ref{theorem: probability within margin})
\item
\emph{(Disagreement outside Margin)} For any absolute constant $c_1>0$ and any function $f(s,n)$, there exists a function $f_4(s,n)>0$ such that $\Pr_{x\sim\mathcal{D}}[\mathrm{sign}(u\cdot x)\not=\mathrm{sign}(v\cdot x)\ \text{and}\ |v\cdot x|\ge f_4(s,n)\theta(u,v)]\le c_1f(s,n)\theta(u,v)$. (Theorem \ref{theorem: disagreement outside band})
\item
\emph{(Variance in 1-D Direction)} There is a function $d(s, n)$ such that for any unit vectors $u$ and $a$ in $\R^n$ such that $\|u-a\|\le r$ and for any $0<t\le d(s, n)$, we have $\E_{x\sim \cD_{u,t}}[(a\cdot x)^2]\le f_5(s, n)(r^2+t^2)$, where $\cD_{u,t}$ is the conditional distribution of $\cD$ over the set $\{x:\ |u\cdot x|\le t\}$. (Theorem \ref{theorem: 1-D variance})
\item
\emph{(Tail Probability)}
We have
$\Pr[\|x\|>\sqrt{n}t]\le \left[1-\frac{cst}{1+ns}\right]^{(1+ns)/s}$. (Theorem \ref{theorem: tail probability})
\end{enumerate}

If $s\rightarrow 0$ (i.e., the distribution is log-concave), then $\gamma(s,m)\rightarrow 0$ and the functions $f(s,n)$, $f_1(s,n)$, $f_2(s,n)$, $f_3(s,n)$, $f_4(s,n)$, $f_5(s,n)$, and $d(s,n)$ are all absolute constants.
\end{theorem}

To prove Theorem \ref{theorem: geometry of s-concavity}, we introduce multiple new techniques, e.g., \emph{extension of Prekopa-Leindler theorem} and \emph{reduction to baseline function} (see the supplementary material for our techniques), which might be of independent interest to optimization more broadly.

\medskip
\noindent{\textbf{Margin Based Active Learning:}} We apply our structural results to margin-based active learning of a halfspace $w^*$ under any isotropic $s$-concave distribution for both \emph{realizable} and \emph{adversarial} noise models. In the realizable case, the instance $X$ is drawn from an isotropic $s$-concave distribution and the label $Y=\sign(w^*\cdot X)$. In the adversarial noise model, an adversary can corrupt any $\eta\ (\le O(\epsilon))$ fraction of labels. For both cases, we show that there exists a \emph{computationally efficient} algorithm that outputs a linear separator $w_T$ such that $\Pr_{x\sim \cD}[\sign(w_T\cdot x)\not=\sign(w^*\cdot x)]\le\epsilon$ (see Theorems \ref{theorem: margin-based active learning realizable case} and \ref{theorem: margin-based active learning adversarial noise}). The label complexity w.r.t. $1/\epsilon$ improves exponentially over the passive learning scenario under $s$-concave distributions, though the underlying distribution might be fat-tailed. To the best of our knowledge, this is the first result concerning the \emph{computationally-efficient, noise-tolerant} margin-based active learning \emph{under the broader class of $s$-concave distributions}. Our work solves an open problem proposed by Awasthi et al.~\cite{awasthi2017power} about exploring wider classes of distributions for provable active learning algorithms.

\medskip
\noindent{\textbf{Disagreement Based Active Learning:}} We apply our results to agnostic disagreement-based active learning under $s$-concave distributions. The key to the analysis is estimating the disagreement coefficient, a distribution-dependent measure of complexity that is used to analyze certain types of active learning algorithms, e.g., the CAL~\cite{cohn1994improving} and $A^2$ algorithm~\cite{balcan2009agnostic}. We work out the disagreement coefficient under isotropic $s$-concave distribution (see Theorem \ref{theorem: disagreement coefficient}). By composing it with the existing work on active learning~\cite{dasgupta2007general}, we obtain a bound on label complexity under the class of $s$-concave distributions. As far as we are aware, this is the first result concerning disagreement-based active learning that goes beyond log-concave distributions. Our bounds on the disagreement coefficient match the best known results for the much less general case of log-concave distributions~\cite{balcan2013active}; Furthermore, they apply to the $s$-concave case where we allow an arbitrary number of discontinuities, a case not captured by \cite{friedman2009active}.

\medskip
\noindent{\textbf{Learning Intersections of Halfspaces:}} Baum's algorithm is one of the most famous algorithms for learning the intersections of halfspaces. The algorithm was first proposed by Baum~\cite{baum1990polynomial} under symmetric distribution, and later extended to log-concave distribution by Klivans et al.~\cite{klivans2009baum} as these distributions are almost symmetric. In this paper, we show that approximate symmetry also holds for the case of $s$-concave distributions. With this, we work out the label complexity of Baum's algorithm under the broader class of $s$-concave distributions (see Theorem \ref{theorem: Baum's algorithm}), and advance the state-of-the-art results (see Table \ref{table: comparisons of distributions}).

We provide lower bounds to partially show the tightness of our analysis. Our results can be potentially applied to other provable learning algorithms as well~\cite{kane2017active,yan2017revisiting,beygelzimer2009importance,xu2017noise,balcan2016noise}, which might be of independent interest. We discuss our techniques and other related papers in the supplementary material.
\begin{table}[t]
\caption{Comparisons with prior distributions for margin-based active learning, disagreement-based active learning, and Baum's algorithm.}
\label{table: comparisons of distributions}
\centering
\begin{tabular}{c|ccc}%
\hline
 & \multicolumn{2}{c}{Prior Work} & Ours\\
\hline
Margin (Efficient, Noise) & uniform~\cite{awasthi2014power} & log-concave~\cite{awasthi2017power} &  $s$-concave\\
Disagreement & uniform~\cite{hanneke2007bound} & nearly-log-concave~\cite{balcan2013active} & $s$-concave\\
Baum's & symmetric~\cite{baum1990polynomial} & log-concave~\cite{klivans2009baum} & $s$-concave\\
\hline
\end{tabular}
\end{table}

\subsection{Our Techniques}

In this section, we introduce the techniques used for obtaining our results.

\begin{wrapfigure}{R}{4cm}
\includegraphics[width=4cm]{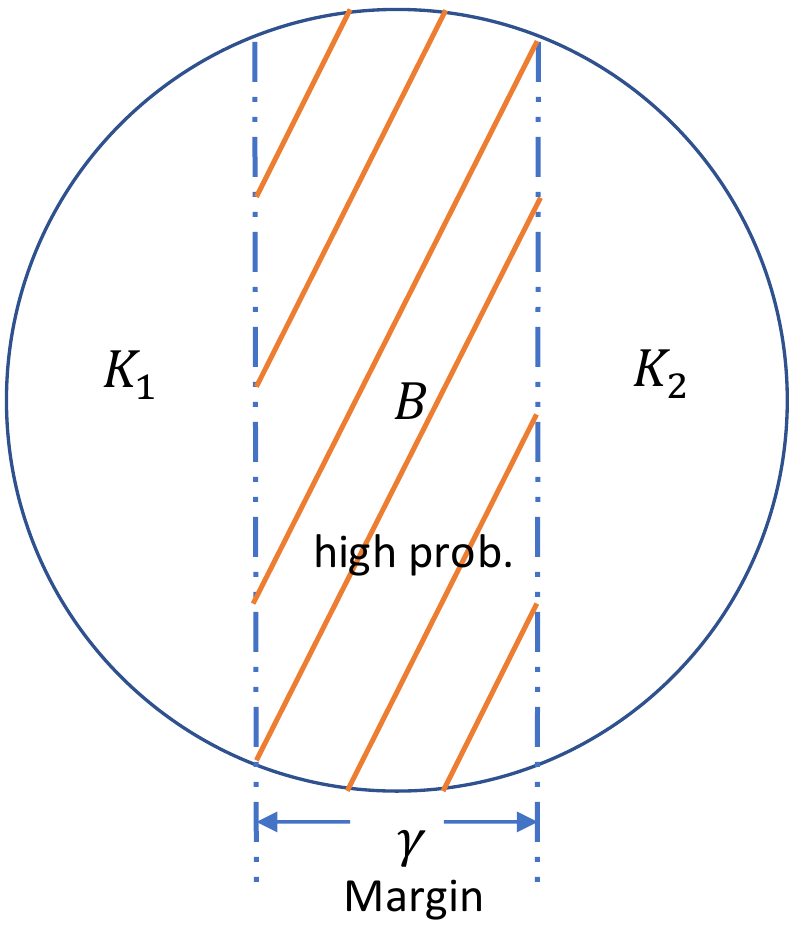}
\caption{Isoperimetry.}
\label{figure: isoperimetry}
\end{wrapfigure}
\medskip
\noindent{\textbf{Marginalization:}} Our results are inspired by isoperimetric inequality for $s$-concave distributions by the work of Chandrasekaran et al.~\cite{chandrasekaran2009sampling}. Roughly, the isoperimetry states that if two sets $K_1$ and $K_2$ are well-separated, then the area $B$ between them has large measure \emph{relative to the measure of the two sets} (see Figure \ref{figure: isoperimetry}). Results of this kind are particularly useful for margin-based active learning of halfspace~\cite{awasthi2014power,awasthi2016learning,awasthi2017power}: The algorithm proceeds in rounds, aiming to cut down the error by half in each round in the band. Since the measure of the band is large or even dominates, the error over the whole space decreases almost by half in each round, resulting in exponentially fast convergence rate. However, in order to make the analysis of such algorithms work for $s$-concave distribution, we typically require more refined geometric properties than the isoperimetry as the isoperimetric inequality states nothing about the \emph{absolute} measure of band under $s$-concave distributions.

The insight behind the isoperimetry is a collection of properties concerning the geometry of probability density. While the geometric properties of some classic paradigms, such as log-concave distributions (for the case of $s=0$), are well-studied~\cite{lovasz2007geometry}, it is typically hard to generalize those results to the $s$-concave distribution, for broader range of $s<0$. This is due to the fact that the class of $s$-concave functions is not closed under marginalization: The marginal of an $s$-concave function may not be $s$-concave any more. This directly restricts the possibility of applying the prior proof techniques for log-concave distribution to the $s$-concave one. Furthermore, previous proofs heavily depend on the assumption that the density is light-tailed (see Theorem 11 in \cite{balcan2013active}), which is not applicable for probably fat-tailed $s$-concave distribution.

To mitigate the above concerns, we begin with a powerful tool from convex geometry by Brascamp and Lieb~\cite{brascamp2002extensions}. This result can be viewed as an extension of celebrated Pr\'ekopa-Leindler inequality, an integral inequality that is closely related to a number of classical inequalities in analysis and serves as the building block of isoperimetry under the log-concave distributions~\cite{caramanis2004inequality,caramanis2007inequality}. With this, we can show that the marginal of any $s$-concave function is $\gamma$-concave, with a closed-form $\gamma$ that is related to the parameter $s$ and the dimension of marginalization. Our analysis is tight as there exists an $s$-concave function with a $\gamma$-concave marginal.

\medskip
\noindent{\textbf{Reduction to 1-D Baseline Function:}} It is in general hard to study a high-dimensional $s$-concave distribution. Instead, we build on the marginalization technique described above to reduce each $n$-dimensional $s$-concave function to the one-dimensional case. Thus it suffices to investigate the geometry of one-dimensional $\gamma$-concave functions. But there are still infinitely many such functions in this class.

Our proofs take a novel analysis by reducing \emph{all} one-dimensional $\gamma$-concave density to a certain baseline function. The baseline function should meet two goals: (a) It represents the worst case in the class of $\gamma$-concave functions, namely, such functions should achieve the bounds of geometric properties of our interest; (b) The function should be easy to analyze, e.g., with closed-form moments or integrations. Note that choosing a baseline function at the ``boundary'' between $\gamma$-concavity and non-$\gamma$-concavity classes readily achieves goal (a). To achieve goal (b), we set the ``template'' function as easy as $h(t)=\alpha(1+\beta t)^{1/\gamma}$ for a particular choice of parameters $\alpha$ and $\beta$. Such functions have many good properties that one can exploit. First, the moments can be represented in closed-form by the beta function. This enables us to figure out the relations among moments of various orders explicitly and obtain a recursive inequality, which is critical for deducing the bounds of one-dimensional geometric properties. Second, $h(t)$ is at the ``boundary'' of $\gamma$-concave class: $h(t)^\eta$ is not a concave function for any $\eta<\gamma$. Therefore, this enables us to analyze the whole class of $s$-concavity by focusing on $h(t)$. Below, we summarize our high-level proof ideas briefly.
\begin{figure}[ht]
\centering
\includegraphics[scale=0.8]{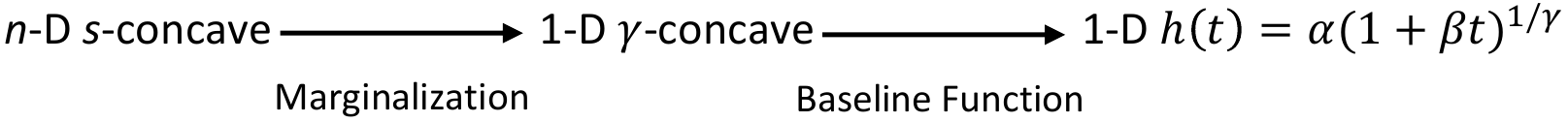}
\end{figure}

\section{Related Work}
\label{section: additional related work}

\noindent{\textbf{Active Learning of Halfspace under Uniform Distribution:}} Learning halfspace has been extensively studied in the past decades~\cite{blum1996polynomial,klivans2002learning,dunagan2004simple,kalai2008agnostically,shalev2010learning,kearns1994introduction,kearns1994toward,kearns1993learning}. Probably one of the most famous results is the VC argument. Vapnik~\cite{vapnik1982estimations} and Blumer et al.~\cite{blumer1989learnability} showed that any hypothesis that is consistent with $\widetilde O(n/\epsilon)$ labeled examples has error at most $\epsilon$, if the VC dimension of the hypothesis class is $n$. The algorithm works under any data distribution and runs in polynomial time when the consistent hypothesis can be found efficiently, e.g., by linear programming in the realizable case. Other algorithms such as Perception~\cite{minsky1987perceptrons}, Winnow~\cite{littlestone1988learning}, and Support Vector Machine~\cite{vapnik1998statistical} provide better guarantees if the target vector has low $\ell_1$ or $\ell_2$ norm. All these results form the basis of passive learning.

To explore the possibility of further improving the label complexity, several algorithms were later proposed in the active learning literature~\cite{beygelzimer2010agnostic,beygelzimer2016search} under the uniform distributions~\cite{dasgupta2005analysis,freund1993information},
among which disagreement-based active learning and margin-based active learning are two typical approaches. In the disagreement-based active learning, the algorithm proceeds in rounds, requesting the labels of instances in the disagreement region among the current candidate hypothesises.
Cohn et al.~\cite{cohn1994improving} provided the first disagreement-based active learning algorithm in the realizable case. Balcan et al.~\cite{balcan2009agnostic} later extended such an algorithm to the agnostic setting by estimating the confidence interval of disagreement region. The analysis technique was further generalized thanks to Hanneke~\cite{hanneke2007bound} by introducing the concept of disagreement coefficient, which is a new measure of complexity for active learning problems and serves as an important element for bounding the label complexity. However, this seminal work only focused on the disagreement coefficient under the uniform distribution.

Margin-based active learning is another line of research in the active learning literature. The algorithm proceeds in rounds, requesting labels of examples aggressively in the margin area around the current guess of hypothesis. Balcan et al.~\cite{balcan2007margin} first proposed an algorithm for margin-based active learning under the uniform distribution in the realizable case. They also provided guarantees under the \emph{Tsybakov noise} model~\cite{tsybakov2004optimal}, but the algorithm is inefficient. To mitigate the issue, Awasthi et al.~\cite{awasthi2015efficient} considered a subclass of Tsybakov noise --- \emph{Massart noise}~\cite{bousquet2005theory}. The algorithm runs in polynomial time by doing a sequence of hinge loss minimizations on the labeled instances. However, it was not clear then whether the analysis works for other distributions instead of the uniform one.

\medskip
\noindent{\textbf{Geometry of Log-Concave Distribution:}} Log-concave distribution, a class of probability distributions such that the logarithm of density function is concave, is a common generalization of uniform distribution over the convex set~\cite{lovasz2007geometry}. Bertsimas and Vempala~\cite{bertsimas2004solving} and Kalai and Vempala~\cite{kalai2006simulated} noticed that efficient sampling, integration, and optimization algorithms for this distribution class rely heavily on the good isoperimetry of density functions. Informally, a function has good isoperimetry if one cannot remove a small-measure set from its domain and partition the domain into two disjoint large-measure sets. The isoperimetry is commonly believed as a characteristic of good geometric properties. To see this, Lov{\'a}sz and Vempala~\cite{lovasz2007geometry} proved the isoperimetric inequality for the log-concave distribution, and provided a bunch of refined geometric properties for this distribution class. Going slightly beyond the log-concave distribution, Caramanis and Mannor~\cite{caramanis2007inequality} showed good isoperimetry for \emph{nearly log-concave} distributions, but more refined geometry was not provided there.

Active learning of halfspace under (nearly) log-concave distribution has a natural connection to the geometry of that distribution (a.k.a. admissible distribution). The connection was first introduced by \cite{balcan2013active}, and is sufficient for the success of disagreement-based and margin-based active learning under log-concave distribution~\cite{balcan2013active}. To resolve the computational issue, Awasthi et al.~\cite{awasthi2014power} studied the probability of disagreement outside the margin under the log-concave distribution, and proposed an efficient algorithm for the challenging adversarial noise. More recently, Awasthi et al.~\cite{awasthi2016learning} provided stronger guarantees for efficient learning of halfspace in the Massart noise model under log-concave distribution.

\medskip
\noindent{\textbf{S-Concave Distribution:}} The problem of extending the log-concave distribution to the broader one for provable learning algorithms has received significant attention in recent years. Although some efforts have been devoted to generalizing the probability distribution, e.g., to the nearly log-concave distribution~\cite{balcan2013active}, the analysis is intrinsically built upon the geometry of log-concave distribution. Moreover, to the best of our knowledge, there is no \emph{efficient, noise-tolerant} active learning algorithm that goes beyond the log-concave distribution. As a candidate extension, the class of $s$-concave distributions has many appealing properties that one can exploit~\cite{chandrasekaran2009sampling,han2016approximation}: (a) The distribution class is much broader than the log-concave distributions as $s=0$ implies the log-concavity; (b) The $s$-concave function mapping from $\R^n$ to $\R_+$ has good isoperimetry if $s\ge -1/(n-1)$; (c) Efficient sampling, integration, and optimization algorithms are available for such distribution class. All these properties inspire our work.

\section{Preliminary}

Before proceeding, we define some notations and clarify our problem setup in this section.

\medskip
\noindent{\textbf{Notations:}}
We will use capital or lower-case letters to represent random variables, $\cD$ to represent an $s$-concave distribution, and $\cD_{u,t}$ to represent the conditional distribution of $\cD$ over the set $\{x:\ |u\cdot x|\le t\}$. We define the \emph{sign} function as $\sign(x)=+1$ if $x\ge 0$ and $-1$ otherwise. We denote by $B(\alpha,\beta)=\int_0^1 t^{\alpha-1}(1-t)^{\beta-1}dt$ the beta function, and $\Gamma(\alpha)=\int_0^\infty t^{\alpha-1}e^{-t}dt$ the gamma function. We will consider a single norm for the vectors in $\R^n$, namely, the $2$-norm denoted by $\|x\|$. We will frequently use $\mu$ (or $\mu_f$, $\mu_\cD$) to represent the measure of the probability distribution $\cD$ with density function $f$. The notation $\ball(w^*,t)$ represents the set $\{w\in\R^n: \|w-w^*\|\le t\}$. For convenience, the symbol $\oplus$ slightly differs from the ordinary addition $+$: For $f=0$ or $g=0$, $\{f^s\oplus g^s\}^{1/s}=0$; Otherwise, $\oplus$ and $+$ are the same. For $u,v\in\R^n$, we define the angle between them as $\theta(u,v)$.

\subsection{From Log-Concavity to S-Concavity}
We begin with the definition of $s$-concavity. There are slight differences among the definitions of $s$-concave density, $s$-concave distribution, and $s$-concave measure.
\begin{definition}[S-Concave (Density) Function, Distribution, Measure]
\label{definition: s-concavity}
A function $f$: $\R^n\rightarrow \R_+$ is $s$-concave, for $-\infty\le s\le 1$, if
$f(\lambda x+(1-\lambda)y)\ge (\lambda f(x)^s+(1-\lambda)f(y)^s)^{1/s}$
for all $\lambda\in[0,1]$, $\forall x,y\in\R^n$.\footnote{When $s\rightarrow 0$, we note that $\lim_{s\rightarrow 0}(\lambda f(x)^s+(1-\lambda)f(y)^s)^{1/s}=\exp(\lambda \log f(x)+(1-\lambda)\log f(y))$. In this case, $f(x)$ is known to be log-concave.}
A probability distribution $\cD$ is $s$-concave, if its density function is $s$-concave.
A probability measure $\mu$ is $s$-concave if
$\mu(\lambda A+(1-\lambda)B)\ge[\lambda\mu(A)^s+(1-\lambda)\mu(B)^s]^{1/s}$
for any sets $A,B\subseteq\R^n$ and $\lambda\in[0,1]$.
\end{definition}

Special classes of $s$-concave functions include \emph{concavity} $(s=1)$, \emph{harmonic-concavity} $(s=-1)$, \emph{quasi-concavity} $(s=-\infty)$, etc. The conditions in Definition \ref{definition: s-concavity} are progressively weaker as $s$ becomes smaller: $s_1$-concave densities (distributions, measures) are $s_2$-concave if $s_1\ge s_2$. Thus one can verify~\cite{chandrasekaran2009sampling}:
$\mbox{concave }(s=1)\subsetneq \mbox{log-concave }(s=0)\subsetneq s\mbox{-concave }(s<0)\subsetneq\mbox{quasi-concave }(s=-\infty)$.

\comment{
In this paper, we investigate refined geometric properties for $s$-concave distribution for $0\ge s\ge-1/(2n+3)$. We note that $s=-1/(n-1)$ is the best possible value for our analysis, as there is significant evidence to indicate that when $s<-1/(n-1)$ good isoperimetry does not hold, which is a necessary condition for our refined geometric properties. We summarize our results in Figure \ref{figure: our results}.

\begin{figure}[ht]
\centering
\includegraphics[scale=0.8]{sconcave.pdf}
\caption{Comparison of our results with prior work and a lower bound.}
\label{figure: our results}
\end{figure}
}

\section{Structural Results of S-Concave Distributions: A Toolkit}
In this section, we develop geometric properties of $s$-concave distribution. The challenge is that unlike the commonly used distributions in learning (uniform or more generally log-concave distributions), this broader class is not closed under the marginalization operator and many such distributions are fat-tailed. To address this issue, we introduce several new techniques. We first introduce the extension of the Prekopa-Leindler inequality so as to reduce the high-dimensional problem to the one-dimensional case. We then reduce the resulting one-dimensional $s$-concave function to a well-defined baseline function, and explore the geometric properties of that baseline function.

\subsection{Marginal Distribution and Cumulative Distribution Function}
We begin with the analysis of the marginal distribution, which forms the basis of other geometric properties of $s$-concave distributions $(s\le 0)$. Unlike the (nearly) log-concave distribution where the marginal remains (nearly) log-concave, the class of $s$-concave distributions is not closed under the marginalization operator. To study the marginal, our primary tool is the theory of convex geometry. Specifically, we will use an extension of the Pr{\'e}kopa-Leindler inequality developed by Brascamp and Lieb~\cite{brascamp2002extensions}, which allows for a characterization of the integral of $s$-concave functions.
\begin{theorem}[\cite{brascamp2002extensions}, Thm 3.3]
\label{theorem: generalized prekopa-leindler inequality}
Let $0<\lambda<1$, and $H_s$, $G_1$, and $G_2$ be non-negative integrable functions on $\mathbb{R}^m$ such that $H_s(\lambda x+(1-\lambda) y)\ge [\lambda G_1(x)^s\oplus(1-\lambda) G_2(y)^s]^{1/s}$ for every $x,y\in\mathbb{R}^m$. Then
$
\int_{\mathbb{R}^m}H_s(x)dx\ge \left[\lambda\left(\int_{\mathbb{R}^m}G_1(x)dx\right)^\gamma+(1-\lambda)\left(\int_{\mathbb{R}^m}G_2(x)dx\right)^\gamma\right]^{1/\gamma}
$ for $s\ge -1/m$,
with $\gamma=s/(1+ms)$.
\end{theorem}

Building on this, the following theorem plays a key role in our analysis of the marginal distribution.
\begin{theorem}[Marginal]
\label{theorem: marginal}
Let $f(x,y)$ be an $s$-concave density on a convex set $K\subseteq \mathbb{R}^{n+m}$ with $s\ge -\frac{1}{m}$. Denote by $K|_{\mathbb{R}^n}=\{x\in\R^n: \exists y\in\R^m\ \mbox{s.t.}\ (x,y)\in K\}$. For every x in $K|_{\mathbb{R}^n}$, consider the section $K(x)\triangleq\{y\in\mathbb{R}^m: (x,y)\in K\}$. Then the marginal density $g(x)\triangleq\int_{K(x)}f(x,y)dy$ is $\gamma$-concave on $K|_{\mathbb{R}^n}$, where $\gamma=\frac{s}{1+ms}$. Moreover, if $f(x,y)$ is isotropic, then $g(x)$ is isotropic.
\end{theorem}

Similar to the marginal, the CDF of an $s$-concave distribution might not remain in the same class. This is in sharp contrast to log-concave distributions. The following theorem studies the CDF of an $s$-concave distribution.
\begin{theorem}
\label{theorem: distribution function}
The CDF of $s$-concave distribution in $\mathbb{R}^n$ is $\gamma$-concave, where $\gamma=\frac{s}{1+ns}$ and $s\ge -\frac{1}{n}$.
\end{theorem}
Theorem \ref{theorem: marginal} and \ref{theorem: distribution function} serve as the bridge that connects high-dimensional $s$-concave distributions to one-dimensional $\gamma$-concave distributions. With them, we are able to reduce the high-dimensional problem to the one-dimensional one.

\subsection{Fat-Tailed Density}
Tail probability is one of the most distinct characteristics of $s$-concave distributions compared to (nearly) log-concave distributions. While it can be shown that the (nearly) log-concave distribution has an exponentially small tail~(Theorem 11, \cite{balcan2013active}), the tail of an $s$-concave distribution is fat, as proved by the following theorem.
\begin{theorem}[Tail Probability]
\label{theorem: tail probability}
Let $x$ come from an isotropic distribution over $\R^n$ with an $s$-concave density. Then for every $t\ge 16$, we have
$\Pr[\|x\|>\sqrt{n}t]\le \left[1-\frac{cst}{1+ns}\right]^{(1+ns)/s}$,
where $c$ is an absolute constant.
\end{theorem}

Theorem \ref{theorem: tail probability} is almost tight for $s<0$. To see this, consider $X$ that is drawn from a one-dimensional Pareto distribution with density $f(x)=(-1-\frac{1}{s})^{-\frac{1}{s}}x^{\frac{1}{s}}$ $(x\ge \frac{s+1}{-s})$. It can be easily seen that $\Pr[X>t]=\left[\frac{-s}{s+1}t\right]^{\frac{s+1}{s}}$ for $t\ge \frac{s+1}{-s}$, which matches Theorem \ref{theorem: tail probability} up to an absolute constant factor.

\subsection{Geometry of S-Concave Distributions}
We now investigate the geometry of $s$-concave distributions. We first consider one-dimensional $s$-concave distributions: We provide bounds on the density of centroid-centered halfspaces (Lemma \ref{lemma: halfspace probability}) and range of the density function (Lemma \ref{lemma: bound for 1 dim}). Building upon these, we develop geometric properties of high-dimensional $s$-concave distributions by reducing the distributions to the one-dimensional case based on marginalization (Theorem \ref{theorem: marginal}).

\subsubsection{One-Dimensional Case}
We begin with the analysis of one-dimensional halfspaces. To bound the probability, a normal technique is to bound the centroid region and the tail region separately. However, the challenge is that the $s$-concave distribution is fat-tailed (Theorem \ref{theorem: tail probability}). So while the probability of a one-dimensional halfspace is bounded below by an absolute constant for log-concave distributions, such a probability for $s$-concave distributions decays as $s\ (\le 0)$ becomes smaller.
The following lemma captures such an intuition.
\begin{lemma}[Density of Centroid-Centered Halfspaces]
\label{lemma: halfspace probability}
Let $X$ be drawn from a one-dimensional distribution with $s$-concave density for $-1/2\le s\le 0$. Then
$\Pr(X\ge \E X)\ge (1+\gamma)^{-1/\gamma}$ for $\gamma=s/(1+s)$.
\end{lemma}

\comment{
We then move on to the study of moments.
\begin{lemma}[Moments]
\label{lemma: moment property}
Let $g:\R_+\rightarrow\R_+$ be an integrable function. Define
$
M_n(g)=\int_0^\infty t^ng(t)dt,
$
and suppose it exists.
Then

(a) $M_n(g)M_{n+2}(g)\ge M_{n+1}(g)^2$ for any $n\in N$.

(b) If $g$ is monotone decreasing, then the sequence $M_n'(g)$, which equals to $nM_{n-1}(g)$ if $n>0$ and $g(0)$ if $n=0$,
obeys $M_n(g)'M_{n+2}(g)'\ge M_{n+1}(g)'$.

(c) If $g$ is $s$-concave ($s>-1/(n+1)$), then the sequence $T_n(g)\triangleq M_n(g)/B(-1/s-n-1,n+1)$ obeys $T_n(g)T_{n+2}(g)\le T_{n+1}(g)^2$ for any $n\in N$.

(d) If $g$ is $s$-concave, then
$g(0)M_1(g)\le M_0(g)^2\frac{1+s}{1+2s}$.
\end{lemma}

\begin{proofoutline}
The proofs of Parts (a) and (b) are from \cite{lovasz2007geometry}. For (c), we novelly choose a baseline $s$-concave function $h(t)=\alpha(1+\beta t)^{1/s}$ which is at the ``boundary" between the family of $s$-concave function and that of the non $s$-concave function. We show that $h$ satisfies the equation
$T_n(h)T_{n+2}(h)=T_{n+1}^2(h)$ for a particular choice of $\alpha$ and $\beta$, because the integral of $h(t)$ is an beta function.
Then by the facts that $h$ is at the ``boundary" and $g$ is any $s$-concave function, we have
$
T_{n+1}(h)\le T_{n+1}(g).
$
The conclusion follows from the above two equations, as well as our choice of $h$ such that $T_n(h)=T_n(g)$ and $T_{n+2}(h)=T_{n+2}(g)$ by adjusting the $\alpha$ and the $\beta$ appropriately. The proof of (d) is similar to that of (c).
\end{proofoutline}
}

We also study the image of a one-dimensional $s$-concave density. The following condition for $s> -1/3$ is for the existence of second-order moment.
\begin{lemma}
\label{lemma: bound for 1 dim}
Let $g:\mathbb{R}\rightarrow\mathbb{R}_+$ be an isotropic $s$-concave density function and $s> -1/3$.
(a) For all $x$, $g(x)\le \frac{1+s}{1+3s}$;
(b) We have $g(0)\ge\sqrt{\frac{1}{3(1+\gamma)^{3/\gamma}}}$, where $\gamma=\frac{s}{s+1}$.
\end{lemma}

\subsubsection{High-Dimensional Case}
We now move on to the high-dimensional case $(n\ge 2)$. In the following, we will assume $-\frac{1}{2n+3}\le s\le 0$. Though this working range of $s$ vanishes as $n$ becomes larger, it is almost the broadest range of $s$ that we can hopefully achieve: Chandrasekaran et al.~\cite{chandrasekaran2009sampling} showed a lower bound of $s\ge -\frac{1}{n-1}$ if one require the $s$-concave distribution to have good geometric properties. In addition, we can see from Theorem \ref{theorem: marginal} that if $s<-\frac{1}{n-1}$, the marginal of an $s$-concave distribution might even not exist; Such a case does happen for certain $s$-concave distributions with $s<-\frac{1}{n-1}$, e.g., the Cauchy distribution. So our range of $s$ is almost tight up to a $1/2$ factor.

We start our analysis with the density of centroid-centered halfspaces in high-dimensional spaces.
\begin{lemma}[Density of Centroid-Centered Halfspaces]
\label{lemma: probability of halfspace}
Let $f:\R^n\rightarrow\R_+$ be an $s$-concave density function, and let $H$ be any halfspace containing its centroid. Then
$\int_H f(x)dx\ge (1+\gamma)^{-1/\gamma}$
for $\gamma=s/(1+ns)$.
\end{lemma}
\begin{proof}
W.L.O.G., we assume $H$ is orthogonal to the first axis. By Theorem \ref{theorem: marginal}, the first marginal of $f$ is $s/(1+(n-1)s)$-concave. Then by Lemma \ref{lemma: halfspace probability},
$\int_H f(x)dx\ge (1+\gamma)^{-1/\gamma}$, where $\gamma=s/(1+ns)$.
\end{proof}

The following theorem is an extension of Lemma \ref{lemma: bound for 1 dim} to high-dimensional spaces. The proofs basically reduce the $n$-dimensional density to its first marginal by Theorem \ref{theorem: marginal}, and apply Lemma \ref{lemma: bound for 1 dim} to bound the image.
\begin{theorem}[Bounds on Density]
\label{theorem: upper and lower bounds}
Let $f:\R^n\rightarrow\R_+$ be an isotropic $s$-concave density. Then

(a) Let $d(s,n)=(1+\gamma)^{-1/\gamma}\frac{1+3\beta}{3+3\beta}$, where $\beta=\frac{s}{1+(n-1)s}$ and $\gamma=\frac{\beta}{1+\beta}$. For any $u\in\R^n$ such that $\|u\|\le d(s,n)$, we have \red{$f(u)\ge \left(\frac{\|u\|}{d}((2-2^{-(n+1)s})^{-1}-1)+1\right)^{1/s}f(0)$}.

(b) $f(x)\le f(0)\left[\left(\frac{1+\beta}{1+3\beta}\sqrt{3(1+\gamma)^{3/\gamma}}2^{n-1+1/s}\right)^{s}-1\right]^{1/s}$ for every $x$.

(c) There exists an $x\in\R^n$ such that $f(x)>(4e\pi)^{-n/2}$.

(d) $(4e\pi)^{-n/2}\left[\left(\frac{1+\beta}{1+3\beta}\sqrt{3(1+\gamma)^{3/\gamma}}2^{n-1+\frac{1}{s}}\right)^{s}-1\right]^{-\frac{1}{s}}<f(0)\le \red{(2-2^{-(n+1)s})^{1/s}\frac{n\Gamma(n/2)}{2\pi^{n/2}d^n}}$.

(e) \red{$f(x)\le (2-2^{-(n+1)s})^{1/s}\frac{n\Gamma(n/2)}{2\pi^{n/2}d^n}\left[\left(\frac{1+\beta}{1+3\beta}\sqrt{3(1+\gamma)^{3/\gamma}}2^{n-1+1/s}\right)^{s}-1\right]^{1/s}$} for every $x$.

(f) For any line $\ell$ through the origin, $\int_{\ell} f \le \red{(2-2^{-ns})^{1/s}\frac{(n-1)\Gamma((n-1)/2)}{2\pi^{(n-1)/2}d^{n-1}}}$.
\end{theorem}

Theorem \ref{theorem: upper and lower bounds} provides uniform bounds on the density function. To obtain more refined upper bound on the image of $s$-concave densities, we have the following lemma. The proof is built upon Theorem \ref{theorem: upper and lower bounds}.
\begin{lemma}[More Refined Upper Bound on Densities]
\label{lemma: more refined upper bound}
Let $f:\R^n\rightarrow \R_+$ be an isotropic $s$-concave density. Then $f(x)\le \beta_1(n,s)(1-s\beta_2(n,s)\|x\|)^{1/s}$ for every $x\in\R^n$, where
\begin{equation*}
\beta_1(n,s)=\frac{(2-2^{-(n+1)s})^{\frac{1}{s}}}{2\pi^{n/2}d^n}(1-s)^{-1/s}n\Gamma(n/2)\left[\left(\frac{1+\beta}{1+3\beta}\sqrt{3(1+\gamma)^{3/\gamma}}2^{n-1+1/s}\right)^{s}-1\right]^{1/s},
\end{equation*}
\begin{equation*}
\beta_2(n,s)=\frac{2\pi^{(n-1)/2}d^{n-1}}{(n-1)\Gamma((n-1)/2)}(2-2^{-ns})^{-\frac{1}{s}}\frac{[(a(n,s)+(1-s)\beta_1(n,s)^s)^{1+\frac{1}{s}}-a(n,s)^{1+\frac{1}{s}}]s}{\beta_1(n,s)^s(1+s)(1-s)},
\end{equation*}
$a(n,s)=(4e\pi)^{-ns/2}\left[\left(\frac{1+\beta}{1+3\beta}\sqrt{3(1+\gamma)^{3/\gamma}}2^{n-1+1/s}\right)^{s}-1\right]^{-1}$, $\gamma=\frac{\beta}{1+\beta}$,
$\beta=\frac{s}{1+(n-1)s}$, and $d=(1+\gamma)^{-\frac{1}{\gamma}}\frac{1+3\beta}{3+3\beta}$.
\end{lemma}

We also give an \emph{absolute} bound on the measure of band.
\begin{theorem}[Probability inside Band]
\label{theorem: probability within margin}
Let $\mathcal{D}$ be an isotropic $s$-concave distribution in $\R^n$. Denote by $f_3(s,n)=2(1+ns)/(1+(n+2)s)$. Then for any unit vector $w$,
$
\Pr_{x\sim\mathcal{D}}[|w\cdot x|\le t]\le f_3(s,n)t.
$
Moreover, if $t\le d(s,n)\triangleq\left(\frac{1+2\gamma}{1+\gamma}\right)^{-\frac{1+\gamma}{\gamma}}\frac{1+3\gamma}{3+3\gamma}$ where $\gamma=\frac{s}{1+(n-1)s}$, then
$
\Pr_{x\sim\mathcal{D}}[|w\cdot x|\le t]> f_2(s,n)t,
$
where $f_2(s,n)=2(2-2^{-2\gamma})^{-1/\gamma}(4e\pi)^{-1/2}\left(2\left(\frac{1+\gamma}{1+3\gamma}\sqrt{3}\left(\frac{1+2\gamma}{1+\gamma}\right)^{\frac{3+3\gamma}{2\gamma}}\right)^\gamma-1\right)^{-1/\gamma}$.
\end{theorem}

To analyze the problem of learning linear separators, we are interested in studying the disagreement between the hypothesis of the output and the hypothesis of the target. The following theorem captures such a characteristic under $s$-concave distributions.

\begin{theorem}[Probability of Disagreement]
\label{theorem: disagreement and angle}
Assume $\mathcal{D}$ is an isotropic $s$-concave distribution in $\R^n$. Then for any two unit vectors $u$ and $v$ in $\R^n$, we have $d_\mathcal{D}(u,v)=\Pr_{x\sim\mathcal{D}}[\sign(u\cdot x)\not=\sign(v\cdot x)]\ge f_1(s,n)\theta(u,v)$, where $f_1(s,n)=c\red{(2-2^{-3\alpha})^{-\frac{1}{\alpha}}\hspace{-0.1cm}\left[\left(\frac{1+\beta}{1+3\beta}\sqrt{3(1+\gamma)^{3/\gamma}}2^{1+1/\alpha}\right)^{\alpha}\hspace{-0.2cm}-1\right]^{-\frac{1}{\alpha}}}(1+\gamma)^{-2/\gamma}\left(\frac{1+3\beta}{3+3\beta}\right)^2$, $c$ is an absolute constant, $\alpha=\frac{s}{1+(n-2)s}$, $\beta=\frac{s}{1+(n-1)s}$, $\gamma=\frac{s}{1+ns}$.
\end{theorem}

Due to space constraints, all missing proofs are deferred to the supplementary material.

\section{Applications: Provable Algorithms under S-Concave Distributions}

In this section, we show that many algorithms that work under log-concave distributions behave well under $s$-concave distributions by applying the above-mentioned geometric properties. For simplicity, we will frequently use the notations in Theorem \ref{theorem: geometry of s-concavity}.

\subsection{Margin Based Active Learning}

We first investigate margin-based active learning under isotropic $s$-concave distributions in both \emph{realizable} and \emph{adversarial noise} models. The algorithm (see Algorithm \ref{algorithm: margin based active learning adversarial noise}) follows a localization technique: It proceeds in rounds, aiming to cut the error down by half in each round in the margin~\cite{balcan2007margin}.

\begin{algorithm}[h]
\footnotesize
\caption{Margin Based Active Learning under S-Concave Distributions}
\begin{algorithmic}
\label{algorithm: margin based active learning adversarial noise}
\STATE {\bfseries Input:} Parameters $b_k$, $\tau_k$, $r_k$, $m_k$, $\kappa$, and $T$ as in Theorem \ref{theorem: margin-based active learning adversarial noise}.
\STATE {\bfseries 1:} Draw $m_1$ examples from $\cD$, label them and put them into $W$.
\STATE {\bfseries 2:} \ \ \textbf{For} $k=1,2,...,T$
\STATE {\bfseries 3:} \quad Find $v_k\in \ball(w_{k-1},r_k)$ to approximately minimize the hinge loss over $W$ s.t. $\|v_k\|\le 1$:\\
\qquad\ $\ell_{\tau_k}\le \min_{w\in \ball(w_{k-1},r_k)\cap \ball(0,1)} \ell_{\tau_k}(w,W)+\kappa/8$.
\STATE {\bfseries 4:} \quad Normalize $v_k$, yielding $w_k=\frac{v_k}{\|v_k\|}$; Clear the working set $W$.
\STATE {\bfseries 5:} \quad \textbf{While} $m_{k+1}$ additional data points are not labeled
\STATE {\bfseries 6:} \quad \quad Draw sample $x$ from $\cD$.
\STATE {\bfseries 7:} \quad \quad \textbf{If} $|w_k\cdot x|\ge b_k$, reject $x$; \textbf{else} ask for label of $x$ and put into $W$.
\STATE {\bfseries Output:} Hypothesis $w_T$.
\end{algorithmic}
\end{algorithm}

\subsubsection{Relevant Properties of S-Concave Distributions}

The analysis requires more refined geometric properties as below. Theorem \ref{theorem: disagreement outside band} basically claims that the error mostly concentrates in a band, and Theorem \ref{theorem: 1-D variance} guarantees that the variance in any 1-D direction cannot be too large. We defer the detailed proofs to the supplementary material.


\begin{theorem}[Disagreement outside Band]
\label{theorem: disagreement outside band}
Let $u$ and $v$ be two vectors in $\R^n$ and assume that $\theta(u,v)=\theta<\pi/2$. Let $\mathcal{D}$ be an isotropic $s$-concave distribution. Then for any absolute constant $c_1>0$ and any function $f_1(s,n)>0$, there exists a function $f_4(s,n)>0$ such that
$
\Pr_{x\sim\mathcal{D}}[\mathrm{sign}(u\cdot x)\not=\mathrm{sign}(v\cdot x)\ \text{and}\ |v\cdot x|\ge f_4(s,n)\theta]\le c_1f_1(s,n)\theta,
$
where $f_4(s,n)=\frac{4\beta_1(2,\alpha)B(-1/\alpha-3,3)}{-c_1f_1(s,n)\alpha^3\beta_2(2,\alpha)^3}$, $B(\cdot,\cdot)$ is the beta function, $\alpha=s/(1+(n-2)s)$, $\beta_1(2,\alpha)$ and $\beta_2(2,\alpha)$ are given by Lemma \ref{lemma: more refined upper bound}.
\end{theorem}

\begin{theorem}[1-D Variance]
\label{theorem: 1-D variance}
Assume that $\cD$ is isotropic $s$-concave. For $d$ given by Theorem \ref{theorem: upper and lower bounds} (a), there is an absolute $C_0$ such that for all $0<t\le d$ and for all $a$ such that $\|u-a\|\le r$ and $\|a\|\le 1$,
$
\E_{x\sim \cD_{u,t}}[(a\cdot x)^2]\le f_5(s,n)(r^2+t^2),
$
where $f_5(s,n)=16+C_0\frac{8\beta_1(2,\eta)B(-1/\eta-3,2)}{f_2(s,n)\beta_2(2,\eta)^3(\eta+1)\eta^2}$, $(\beta_1(2,\eta), \beta_2(2,\eta))$ and $f_2(s,n)$ are given by Lemma \ref{lemma: more refined upper bound} and Theorem \ref{theorem: probability within margin}, and $\eta=\frac{s}{1+(n-2)s}$.
\end{theorem}

\subsubsection{Realizable Case}

We show that margin-based active learning works under $s$-concave distributions in the realizable case.
\begin{theorem}
\label{theorem: margin-based active learning realizable case}
In the realizable case, let $\cD$ be an isotropic $s$-concave distribution in $\R^n$. Then for $0<\epsilon<1/4$, $\delta>0$, and absolute constants $c$, there is an algorithm (see the supplementary material) that runs in $T=\lceil\log \frac{1}{c\epsilon}\rceil$ iterations, requires $m_k\hspace{-0.1cm}=\hspace{-0.1cm}O\hspace{-0.1cm}\left(\frac{f_3\min\{2^{-k}f_4f_1^{-1},d\}}{2^{-k}}\left(n\log \frac{f_3\min\{2^{-k}f_4f_1^{-1},d\}}{2^{-k}}\hspace{-0.1cm}+\hspace{-0.1cm}\log \frac{1+s-k}{\delta}\right)\hspace{-0.1cm}\right)$ labels in the $k$-th round, and outputs a linear separator of error at most $\epsilon$ with probability at least $1-\delta$. In particular, when $s\rightarrow 0$ (a.k.a. log-concave), we have $m_k=O\left(n+\log(\frac{1+s-k}{\delta})\right)$.
\end{theorem}

By Theorem \ref{theorem: margin-based active learning realizable case}, we see that the algorithm of margin-based active learning under $s$-concave distributions works almost as well as the log-concave distributions in the resizable case, improving exponentially w.r.t. the variable $1/\epsilon$ over passive learning algorithms.

\subsubsection{Efficient Learning with Adversarial Noise}

In the adversarial noise model, an adversary can choose any distribution $\widetilde\cP$ over $\R^n\times \{+1,-1\}$ such that the marginal $\cD$ over $\R^n$ is $s$-concave but an $\eta$ fraction of labels can be flipped adversarially. The analysis builds upon an induction technique where in
each round we do hinge loss minimization in the band and cut down the 0/1 loss by half. The algorithm was previously analyzed in \cite{awasthi2014power,awasthi2017power} for the special class of log-concave distributions. In this paper, we analyze it for the much more general class of $s$-concave distributions.

\begin{theorem}
\label{theorem: margin-based active learning adversarial noise}
Let $\cD$ be an isotropic $s$-concave distribution in $\R^n$ over $x$ and the label $y$ obey the adversarial noise model. If the rate $\eta$ of adversarial noise satisfies $\eta<c_0\epsilon$ for some absolute constant $c_0$, then for $0<\epsilon<1/4$, $\delta>0$, and an absolute constant $c$, Algorithm \ref{algorithm: margin based active learning adversarial noise} runs in $T=\lceil\log \frac{1}{c\epsilon}\rceil$ iterations, outputs a linear separator $w_T$ such that $\Pr_{x\sim \cD}[\sign(w_T\cdot x)\not=\sign(w^*\cdot x)]\le\epsilon$ with probability at least $1-\delta$. The label complexity in the $k$-th round is $m_k=O\left( \frac{[b_{k-1}s+\tau_k(1+ns)[1-(\delta/(\sqrt{n}(k+k^2)))^{s/(1+ns)}]+\tau_ks]^2}{\kappa^2\tau_k^2s^2}n\left(n+\log \frac{k+k^2}{\delta}\right)\right)$, where $\kappa=\max\left\{\frac{f_3\tau_k}{f_2\min\{b_{k-1},d\}},\frac{b_{k-1}\sqrt{f_5}}{\tau_k\sqrt{f_2}}\right\}$, $\tau_k=\Theta\left(f_1^{-2}f_2^{-1/2}f_3f_4^2f_5^{1/2}2^{-(k-1)}\right)$, $b_k=\min\{\Theta(2^{-k}f_4f_1^{-1}),d\}$. In particular, if $s\rightarrow 0$, $m_k=O\left(n\log(\frac{n}{\epsilon\delta})(n+\log(\frac{k}{\delta}))\right)$.
\end{theorem}

By Theorem \ref{theorem: margin-based active learning adversarial noise}, the label complexity of margin-based active learning improves exponentially over that of passive learning w.r.t. $1/\epsilon$ even under fat-tailed $s$-concave distributions and challenging adversarial noise model.

\subsection{Disagreement Based Active Learning}

We apply our results to the analysis of disagreement-based active learning under $s$-concave distributions. The key is estimating the disagreement coefficient, a measure of complexity of active learning problems that can be used to bound the label complexity~\cite{hanneke2007bound}. Recall the definition of the disagreement coefficient w.r.t. classifier $w^*$, precision $\epsilon$, and distribution $\mathcal{D}$ as follows. For any $r>0$, define $\ball_{\mathcal{D}}(w,r)=\{u\in\mathcal{H}:d_{\mathcal{D}}(u,w)\le r\}$ where $d_{\mathcal{D}}(u,w)=\Pr_{x\sim\mathcal{D}}[(u\cdot x)(w\cdot x)<0]$. Define the disagreement region as $\mbox{DIS}(\mathcal{H})=\{x:\exists u,v\in\mathcal{H}\ \mbox{s.t.}\ (u\cdot x)(v\cdot x)<0\}$. Let the Alexander capacity $\mbox{cap}_{w^*,\mathcal{D}}=\frac{\Pr_{\mathcal{D}}(\text{DIS}(\ball_{\mathcal{D}}(w^*,r)))}{r}$. The disagreement coefficient is defined as $\Theta_{w^*,\mathcal{D}}(\epsilon)=\sup_{r\ge\epsilon}[\text{cap}_{w^*,\mathcal{D}}(r)]$. Below, we state our results on the disagreement coefficient under isotropic $s$-concave distributions.

\begin{theorem}[Disagreement Coefficient]
\label{theorem: disagreement coefficient}
Let $\mathcal{D}$ be an isotropic $s$-concave distribution over $\R^n$. For any $w^*$ and $r>0$, the disagreement coefficient is $\Theta_{w^*,\mathcal{D}}(\epsilon)=O\left(\frac{\sqrt{n}(1+ns)^2}{s(1+(n+2)s)f_1(s,n)}(1-\epsilon^{\frac{s}{1+ns}})\right)$. In particular, when $s\rightarrow0$ (a.k.a. log-concave), $\Theta_{w^*,\mathcal{D}}(\epsilon)=O(\sqrt{n}\log(1/\epsilon))$.
\end{theorem}

Our bounds on the disagreement coefficient match the best known results for the much less general case of log-concave distributions~\cite{balcan2013active}; Furthermore, they apply to the $s$-concave case where we allow arbitrary number of discontinuities, a case not captured by \cite{friedman2009active}. The result immediately implies concrete bounds on the label complexity of disagreement-based active learning algorithms, e.g., CAL~\cite{cohn1994improving} and $A^2$~\cite{balcan2009agnostic}. For instance, by composing it with the result from \cite{dasgupta2007general}, we obtain a bound of
$$\widetilde O\left(n^{3/2}\frac{(1+ns)^2}{s(1+(n+2)s)f(s)}(1-\epsilon^{s/(1+ns)})\left(\log^2 \frac{1}{\epsilon}+\frac{OPT^2}{\epsilon^2}\right)\right)$$
for \emph{agnostic} active learning under an isotropic $s$-concave distribution $\cD$. Namely, it suffices to output a halfspace with error at most $OPT+\epsilon$, where $OPT=\min_{w}\err_\cD(w)$.

\subsection{Learning Intersections of Halfspaces}

Baum~\cite{baum1990polynomial} provided a polynomial-time algorithm for learning the intersections of halfspaces w.r.t. symmetric distributions. Later, Klivans~\cite{klivans2009baum} extended the result by showing that the algorithm works under any distribution $\cD$ as long as $\mu_\cD(E)\approx\mu_\cD(-E)$ for any set $E$. In this section, we show that it is possible to learn intersections of halfspaces under the broader class of s-concave distributions.

\begin{theorem}
\label{theorem: Baum's algorithm}
In the PAC realizable case, there is an algorithm (see the supplementary material) that outputs a hypothesis $h$ of error at most $\epsilon$ with probability at least $1-\delta$ under isotropic $s$-concave distributions. The label complexity is $M(\epsilon/2,\delta/4,n^2)+\max\{2m_2/\epsilon,(2/\epsilon^2)\log(4/\delta)\}$, where $M(\epsilon,\delta,m)$ is defined by $M(\epsilon,\delta,n)=O\left(\frac{n}{\epsilon}\log \frac{1}{\epsilon}+\frac{1}{\epsilon}\log\frac{1}{\delta}\right)$, $m_2=M(\max\{\delta/(4eKm_1),\epsilon/2\},\delta/4,n)$, $d=(1+\gamma)^{-1/\gamma}\frac{1+3\beta}{3+3\beta}$, $K=\beta_1(3,\kappa)\frac{B(-1/\kappa-3,3)}{(-\kappa\beta_2(3,\kappa))^3}\frac{3+1/\kappa}{h(\kappa)d^{3+1/\kappa}}$, $\beta=\frac{\kappa}{1+2\kappa}$, $\gamma=\frac{\kappa}{1+\kappa}$, $\kappa=\frac{s}{1+(n-3)s}$, and $h(\kappa)=\left(\frac{1}{d}((2-2^{-4\kappa})^{-1}-1)+1\right)^{\frac{1}{\kappa}}(4e\pi)^{-\frac{3}{2}}\left[\left(\frac{1+\beta}{1+3\beta}\sqrt{3(1\hspace{-0.1cm}+\hspace{-0.1cm}\gamma)^{3/\gamma}}2^{2+\frac{1}{\kappa}}\right)^{\kappa}\hspace{-0.2cm}-\hspace{-0.1cm}1\right]^{-1/\kappa}$. In particular, if $s\rightarrow 0$ (a.k.a. log-concave), $K$ is an absolute constant.
\end{theorem}

\section{Lower Bounds}

In this section, we give information-theoretic lower bounds on the label complexity of passive and active learning of homogeneous halfspaces under $s$-concave distributions.
\begin{theorem}
\label{theorem: lower bounds of learning halfspace}
For a fixed value $-\frac{1}{2n+3}\le s\le 0$ we have: (a) For any $s$-concave distribution $\cD$ in $\R^n$ whose covariance matrix is of full rank, the sample complexity of learning origin-centered linear separators under $\cD$ in the passive learning scenario is $\Omega\left(nf_1(s,n)/\epsilon\right)$; (b) The label complexity of active learning of linear separators under $s$-concave distributions is $\Omega\left(n\log\left(f_1(s,n)/\epsilon\right)\right)$.
\end{theorem}

If the covariance matrix of $\cD$ is not of full rank, then the intrinsic dimension is less than $d$. So our lower bounds essentially apply to all $s$-concave distributions. According to Theorem \ref{theorem: lower bounds of learning halfspace}, it is possible to have an exponential improvement of label complexity w.r.t. $1/\epsilon$ over passive learning by active sampling, even though the underlying distribution is a fat-tailed $s$-concave distribution. This observation is captured by Theorems \ref{theorem: margin-based active learning realizable case} and \ref{theorem: margin-based active learning adversarial noise}.

\section{Conclusions}

In this paper, we study the geometric properties of $s$-concave distributions. Our work advances the state-of-the-art results on the margin-based active learning, disagreement-based active learning, and learning intersections of halfspaces w.r.t. the distributions over the instance space. When $s\rightarrow 0$, our results reduce to the best-known results for log-concave distributions. The geometric properties of $s$-concave distributions can be potentially applied to other learning algorithms, which might be of independent interest more broadly.

\medskip
\noindent \textbf{Acknowledgements.} This work was supported in part by grants  NSF-CCF 1535967, NSF CCF-1422910, NSF CCF-1451177, a Sloan Fellowship, and a Microsoft Research
Fellowship.

\bibliographystyle{abbrv}
\bibliography{reference}

\begin{thebibliography}{10}

\bibitem{anthony2009neural}
M.~Anthony and P.~L. Bartlett.
\newblock {\em Neural network learning: Theoretical foundations}.
\newblock Cambridge University Press, 2009.

\bibitem{applegate1991sampling}
D.~Applegate and R.~Kannan.
\newblock Sampling and integration of near log-concave functions.
\newblock In {\em ACM Symposium on Theory of Computing}, pages 156--163, 1991.

\bibitem{awasthi2015efficient}
P.~Awasthi, M.-F. Balcan, N.~Haghtalab, and R.~Urner.
\newblock Efficient learning of linear separators under bounded noise.
\newblock In {\em Annual Conference on Learning Theory}, pages 167--190, 2015.

\bibitem{awasthi2016learning}
P.~Awasthi, M.-F. Balcan, N.~Haghtalab, and H.~Zhang.
\newblock Learning and 1-bit compressed sensing under asymmetric noise.
\newblock In {\em Annual Conference on Learning Theory}, pages 152--192, 2016.

\bibitem{awasthi2014power}
P.~Awasthi, M.-F. Balcan, and P.~M. Long.
\newblock The power of localization for efficiently learning linear separators
  with noise.
\newblock In {\em ACM Symposium on Theory of Computing}, pages 449--458, 2014.

\bibitem{awasthi2017power}
P.~Awasthi, M.-F. Balcan, and P.~M. Long.
\newblock The power of localization for efficiently learning linear separators
  with noise.
\newblock {\em Journal of the ACM}, 63(6):50, 2017.

\bibitem{balcan2009agnostic}
M.-F. Balcan, A.~Beygelzimer, and J.~Langford.
\newblock Agnostic active learning.
\newblock {\em Journal of Computer and System Sciences}, 75(1):78--89, 2009.

\bibitem{balcan2007margin}
M.-F. Balcan, A.~Broder, and T.~Zhang.
\newblock Margin based active learning.
\newblock In {\em Annual Conference on Learning Theory}, pages 35--50, 2007.

\bibitem{balcan2013active}
M.-F. Balcan and P.~M. Long.
\newblock Active and passive learning of linear separators under log-concave
  distributions.
\newblock In {\em Annual Conference on Learning Theory}, pages 288--316, 2013.

\bibitem{balcan2016noise}
M.-F. Balcan and H.~Zhang.
\newblock Noise-tolerant life-long matrix completion via adaptive sampling.
\newblock In {\em Advances in Neural Information Processing Systems}, pages
  2955--2963, 2016.

\bibitem{baum1990polynomial}
E.~B. Baum.
\newblock A polynomial time algorithm that learns two hidden unit nets.
\newblock {\em Neural Computation}, 2(4):510--522, 1990.

\bibitem{bertsimas2004solving}
D.~Bertsimas and S.~Vempala.
\newblock Solving convex programs by random walks.
\newblock {\em Journal of the ACM}, 51(4):540--556, 2004.

\bibitem{beygelzimer2009importance}
A.~Beygelzimer, S.~Dasgupta, and J.~Langford.
\newblock Importance weighted active learning.
\newblock In {\em International Conference on Machine Learning}, pages 49--56,
  2009.

\bibitem{beygelzimer2016search}
A.~Beygelzimer, D.~J. Hsu, J.~Langford, and C.~Zhang.
\newblock Search improves label for active learning.
\newblock In {\em Advances in Neural Information Processing Systems}, pages
  3342--3350, 2016.

\bibitem{beygelzimer2010agnostic}
A.~Beygelzimer, D.~J. Hsu, J.~Langford, and T.~Zhang.
\newblock Agnostic active learning without constraints.
\newblock In {\em Advances in Neural Information Processing Systems}, pages
  199--207, 2010.

\bibitem{blum1996polynomial}
A.~Blum, A.~Frieze, R.~Kannan, and S.~Vempala.
\newblock A polynomial-time algorithm for learning noisy linear threshold
  functions.
\newblock In {\em IEEE Symposium on Foundations of Computer Science}, pages
  330--338, 1996.

\bibitem{blumer1989learnability}
A.~Blumer, A.~Ehrenfeucht, D.~Haussler, and M.~K. Warmuth.
\newblock Learnability and the {V}apnik-{C}hervonenkis dimension.
\newblock {\em Journal of the ACM}, 36(4):929--965, 1989.

\bibitem{bobkov2007large}
S.~G. Bobkov.
\newblock Large deviations and isoperimetry over convex probability measures
  with heavy tails.
\newblock {\em Electronic Journal of Probability}, 12:1072--1100, 2007.

\bibitem{bousquet2005theory}
O.~Bousquet, S.~Boucheron, and G.~Lugosi.
\newblock Theory of classification: A survey of recent advances.
\newblock {\em ESAIM: Probability and Statistics}, 9(9):323--375, 2005.

\bibitem{brascamp2002extensions}
H.~J. Brascamp and E.~H. Lieb.
\newblock On extensions of the {B}runn-{M}inkowski and {P}r{\'e}kopa-{L}eindler
  theorems, including inequalities for log concave functions, and with an
  application to the diffusion equation.
\newblock {\em Journal of Functional Analysis}, 22(4):366--389, 1976.

\bibitem{caramanis2004inequality}
C.~Caramanis and S.~Mannor.
\newblock An inequality for nearly log-concave distributions with applications
  to learning.
\newblock In {\em Annual Conference on Learning Theory}, pages 534--548, 2004.

\bibitem{caramanis2007inequality}
C.~Caramanis and S.~Mannor.
\newblock An inequality for nearly log-concave distributions with applications
  to learning.
\newblock {\em IEEE Transactions on Information Theory}, 53(3):1043--1057,
  2007.

\bibitem{chandrasekaran2009sampling}
K.~Chandrasekaran, A.~Deshpande, and S.~Vempala.
\newblock Sampling s-concave functions: The limit of convexity based
  isoperimetry.
\newblock In {\em Approximation, Randomization, and Combinatorial Optimization.
  Algorithms and Techniques}, pages 420--433. 2009.

\bibitem{cohn1994improving}
D.~Cohn, L.~Atlas, and R.~Ladner.
\newblock Improving generalization with active learning.
\newblock {\em Machine Learning}, 15(2):201--221, 1994.

\bibitem{daniely2016complexity}
A.~Daniely.
\newblock Complexity theoretic limitations on learning halfspaces.
\newblock In {\em ACM Symposium on Theory of computing}, pages 105--117, 2016.

\bibitem{dasgupta2004analysis}
S.~Dasgupta.
\newblock Analysis of a greedy active learning strategy.
\newblock In {\em Advances in Neural Information Processing Systems},
  volume~17, pages 337--344, 2004.

\bibitem{dasgupta2007general}
S.~Dasgupta, D.~J. Hsu, and C.~Monteleoni.
\newblock A general agnostic active learning algorithm.
\newblock In {\em Advances in Neural Information Processing Systems}, pages
  353--360, 2007.

\bibitem{dasgupta2005analysis}
S.~Dasgupta, A.~T. Kalai, and C.~Monteleoni.
\newblock Analysis of perceptron-based active learning.
\newblock In {\em Annual Conference on Learning Theory}, pages 249--263, 2005.

\bibitem{dunagan2004simple}
J.~Dunagan and S.~Vempala.
\newblock A simple polynomial-time rescaling algorithm for solving linear
  programs.
\newblock In {\em ACM Symposium on Theory of computing}, pages 315--320, 2004.

\bibitem{freund1993information}
Y.~Freund, H.~S. Seung, E.~Shamir, and N.~Tishby.
\newblock Information, prediction, and query by committee.
\newblock {\em Advances in Neural Information Processing Systems}, pages
  483--483, 1993.

\bibitem{friedman2009active}
E.~Friedman.
\newblock Active learning for smooth problems.
\newblock In {\em Annual Conference on Learning Theory}, 2009.

\bibitem{guruswami2009hardness}
V.~Guruswami and P.~Raghavendra.
\newblock Hardness of learning halfspaces with noise.
\newblock {\em SIAM Journal on Computing}, 39(2):742--765, 2009.

\bibitem{han2016approximation}
Q.~Han and J.~A. Wellner.
\newblock Approximation and estimation of s-concave densities via {R}{\'e}nyi
  divergences.
\newblock {\em The Annals of Statistics}, 44(3):1332--1359, 2016.

\bibitem{hanneke2007bound}
S.~Hanneke.
\newblock A bound on the label complexity of agnostic active learning.
\newblock In {\em International Conference on Machine Learning}, pages
  353--360, 2007.

\bibitem{hanneke2014theory}
S.~Hanneke et~al.
\newblock Theory of disagreement-based active learning.
\newblock {\em Foundations and Trends in Machine Learning}, 7(2-3):131--309,
  2014.

\bibitem{kalai2008agnostically}
A.~T. Kalai, A.~R. Klivans, Y.~Mansour, and R.~A. Servedio.
\newblock Agnostically learning halfspaces.
\newblock {\em SIAM Journal on Computing}, 37(6):1777--1805, 2008.

\bibitem{kalai2006simulated}
A.~T. Kalai and S.~Vempala.
\newblock Simulated annealing for convex optimization.
\newblock {\em Mathematics of Operations Research}, 31(2):253--266, 2006.

\bibitem{kane2017active}
D.~M. Kane, S.~Lovett, S.~Moran, and J.~Zhang.
\newblock Active classification with comparison queries.
\newblock In {\em IEEE Symposium on Foundations of Computer Science}, pages
  355--366, 2017.

\bibitem{kearns1993learning}
M.~Kearns and M.~Li.
\newblock Learning in the presence of malicious errors.
\newblock {\em SIAM Journal on Computing}, 22(4):807--837, 1993.

\bibitem{kearns1994toward}
M.~J. Kearns, R.~E. Schapire, and L.~M. Sellie.
\newblock Toward efficient agnostic learning.
\newblock {\em Machine Learning}, 17(2-3):115--141, 1994.

\bibitem{kearns1994introduction}
M.~J. Kearns and U.~V. Vazirani.
\newblock {\em An introduction to computational learning theory}.
\newblock MIT press, 1994.

\bibitem{klivans2014embedding}
A.~Klivans and P.~Kothari.
\newblock Embedding hard learning problems into gaussian space.
\newblock {\em International Workshop on Approximation Algorithms for
  Combinatorial Optimization Problems}, 28:793--809, 2014.

\bibitem{klivans2009learning}
A.~R. Klivans, P.~M. Long, and R.~A. Servedio.
\newblock Learning halfspaces with malicious noise.
\newblock {\em Journal of Machine Learning Research}, 10:2715--2740, 2009.

\bibitem{klivans2009baum}
A.~R. Klivans, P.~M. Long, and A.~K. Tang.
\newblock Baum{'}s algorithm learns intersections of halfspaces with respect to
  log-concave distributions.
\newblock In {\em Approximation, Randomization, and Combinatorial Optimization.
  Algorithms and Techniques}, pages 588--600. 2009.

\bibitem{klivans2002learning}
A.~R. Klivans, R.~O'Donnell, and R.~A. Servedio.
\newblock Learning intersections and thresholds of halfspaces.
\newblock In {\em IEEE Symposium on Foundations of Computer Science}, pages
  177--186, 2002.

\bibitem{kulkarni1993active}
S.~R. Kulkarni, S.~K. Mitter, and J.~N. Tsitsiklis.
\newblock Active learning using arbitrary binary valued queries.
\newblock {\em Machine Learning}, 11(1):23--35, 1993.

\bibitem{littlestone1988learning}
N.~Littlestone.
\newblock Learning quickly when irrelevant attributes abound: A new
  linear-threshold algorithm.
\newblock {\em Machine Learning}, 2(4):285--318, 1988.

\bibitem{long1995sample}
P.~M. Long.
\newblock On the sample complexity of pac learning half-spaces against the
  uniform distribution.
\newblock {\em IEEE Transactions on Neural Networks}, 6(6):1556--1559, 1995.

\bibitem{lovasz2007geometry}
L.~Lov{\'a}sz and S.~Vempala.
\newblock The geometry of logconcave functions and sampling algorithms.
\newblock {\em Random Structures $\&$ Algorithms}, 30(3):307--358, 2007.

\bibitem{minsky1987perceptrons}
M.~Minsky and S.~Papert.
\newblock Perceptrons--extended edition: An introduction to computational
  geometry, 1987.

\bibitem{servedio2001efficient}
R.~A. Servedio.
\newblock {\em Efficient algorithms in computational learning theory}.
\newblock PhD thesis, Harvard University, 2001.

\bibitem{shalev2010learning}
S.~Shalev-Shwartz, O.~Shamir, and K.~Sridharan.
\newblock Learning kernel-based halfspaces with the zero-one loss.
\newblock {\em arXiv preprint arXiv:1005.3681}, 2010.

\bibitem{tsybakov2004optimal}
A.~B. Tsybakov.
\newblock Optimal aggregation of classifiers in statistical learning.
\newblock {\em The Annals of Statistics}, pages 135--166, 2004.

\bibitem{vapnik1982estimations}
V.~Vapnik.
\newblock {\em Estimations of dependences based on statistical data}.
\newblock Springer, 1982.

\bibitem{vapnik1998statistical}
V.~Vapnik.
\newblock {\em The nature of statistical learning theory}.
\newblock Springer Science \& Business Media, 2013.

\bibitem{wang2011smoothness}
L.~Wang.
\newblock Smoothness, disagreement coefficient, and the label complexity of
  agnostic active learning.
\newblock {\em Journal of Machine Learning Research}, 12(Jul):2269--2292, 2011.

\bibitem{xu2017noise}
Y.~Xu, H.~Zhang, A.~Singh, A.~Dubrawski, and K.~Miller.
\newblock Noise-tolerant interactive learning using pairwise comparisons.
\newblock In {\em Advances in Neural Information Processing Systems}, pages
  2428--2437, 2017.

\bibitem{yan2017revisiting}
S.~Yan and C.~Zhang.
\newblock Revisiting perceptron: Efficient and label-optimal active learning of
  halfspaces.
\newblock {\em arXiv preprint arXiv:1702.05581}, 2017.

\bibitem{zhang2014beyond}
C.~Zhang and K.~Chaudhuri.
\newblock Beyond disagreement-based agnostic active learning.
\newblock In {\em Advances in Neural Information Processing Systems}, pages
  442--450, 2014.

\end{thebibliography}

\newpage
\appendix

\section{Proof of Theorem \ref{theorem: marginal}}
\noindent{\textbf{Theorem \ref{theorem: marginal} (restated)}
\emph{Let $f(x,y)$ be an $s$-concave density on a convex set $K\subseteq \mathbb{R}^{n+m}$ with $s\ge -\frac{1}{m}$. Denote by $K|_{\mathbb{R}^n}=\{x\in\R^n: \exists y\in\R^m\ \mbox{s.t.}\ (x,y)\in K\}$. For every x in $K|_{\mathbb{R}^n}$, consider the section $K(x)\triangleq\{y\in\mathbb{R}^m: (x,y)\in K\}$. Then the marginal density $g(x)\triangleq\int_{K(x)}f(x,y)dy$ is $\gamma$-concave on $K|_{\mathbb{R}^n}$, where $\gamma=\frac{s}{1+ms}$. Moreover, if $f(x,y)$ is isotropic, then $g(x)$ is isotropic.}

\begin{proof}
\comment{
Roughly, for any $(x_1,y_1),(x_2,y_2)\in K|_{\mathbb{R}^{n+m}}$ and any point $(x,y)$ such that $(x,y)=\lambda (x_1,y_1)+(1-\lambda)(x_2,y_2)$, by the fact that $f(x,y)$ is $s$-concave we have $f(x,y)\ge (\lambda f(x_1,y_1)^s+(1-\lambda)f(x_2,y_2)^s)^{1/s}$ for $\lambda\in[0,1]$. Let in Theorem \ref{theorem: generalized prekopa-leindler inequality} $G_1(x_1)=f(x_1,y_1)$, $G_2(x_2)=f(x_2,y_2)$, and $H_s(\lambda x_1+(1-\lambda)x_2)=f(\lambda x_1+(1-\lambda)x_2,\lambda y_1+(1-\lambda)y_2)$, the proof is completed.}
The proof that $g(x)$ is isotropic is standard~\cite{lovasz2007geometry}. We now prove the first part. Let $x_1$, $x_2$ be any two points.
Define $g_i(y)=f(x_i,y)$ for $i=1,2$. So the functions $g_i(y)$ is defined on $K(x_i)$, $i=1,2$. Now let $x=\lambda x_1+(1-\lambda)x_2$ for $\lambda\in(0,1)$ and define $h_s(y)=f(x,y)$ on $K(x)$. Notice that for any $y_i\in K(x_i)$, $i=1,2$, $y=\lambda y_1+(1-\lambda)y_2\in K(x)$. To see this, by the convexity of the set $K$, the point
$(x,y)=\lambda(x_1,y_1)+(1-\lambda)(x_2,y_2)$
belongs to $K$. So $y\in K(x)$, i.e., $\lambda K(x_1)+(1-\lambda)K(x_2)\subseteq K(x)$. Using the $s$-concavity of $f(x,y)$, we have
$f(x,y)=f(\lambda(x_1,y_1)+(1-\lambda)(x_2,y_2))\ge \left[\lambda f(x_1,y_1)^s+(1-\lambda)f(x_2,y_2)^s\right]^{1/s}$,
which implies that
$h_s(y)=h_s(\lambda y_1+(1-\lambda)y_2)\ge \left[\lambda g_1(y_1)^s+(1-\lambda)g_2(y_2)^s\right]^{1/s}$.
Denote by $I_{(\cdot)}$ the indicator function. So
$h_s(\lambda y_1+(1-\lambda)y_2)I_{K(x)}(y)\ge \left[\lambda (g_1(y_1)I_{K(x_1)}(y_1))^s\oplus(1-\lambda)(g_2(y_2)I_{K(x_2)}(y_2))^s\right]^{1/s}$.
Set $H_s(y)\hspace{-0.1cm}=\hspace{-0.1cm}h_s(\lambda y_1\hspace{-0.1cm}+\hspace{-0.1cm}(1\hspace{-0.1cm}-\hspace{-0.1cm}\lambda)y_2)I_{K(x)}$, $G_1(y_1)\hspace{-0.1cm}=\hspace{-0.1cm}g_1(y_1)I_{K(x_1)}$, $G_2(y_1)=g_2(y_1)I_{K(x_2)}$. By Theorem \ref{theorem: generalized prekopa-leindler inequality},
\begin{equation*}
\begin{split}
g(x)&=\hspace{-0.1cm}\int_{\mathbb{R}^m}H_s(y)dy=\hspace{-0.1cm}\int_{\mathbb{R}^m}h_s(y)I_{K(x)}(y)dy\ge \hspace{-0.1cm}\left[(1\hspace{-0.1cm}-\hspace{-0.1cm}\lambda)\hspace{-0.1cm}\left(\int_{\mathbb{R}^n}\hspace{-0.2cm}G_1(y)dy\right)^\gamma\hspace{-0.2cm}+\hspace{-0.1cm}\lambda\left(\int_{\mathbb{R}^n}\hspace{-0.2cm}G_2(y)dy\right)^\gamma\right]^{1/\gamma}\\
&=\left[(1-\lambda)\left(\int_{\mathbb{R}^n}f(x_1,y_1)I_{K(x_1)}(y_1)dy_1\right)^\gamma+\lambda\left(\int_{\mathbb{R}^n}f(x_2,y_2)I_{K(x_2)}(y_2)dy_2\right)^\gamma\right]^{1/\gamma}\\
&=\left[(1-\lambda)g(x_1)^\gamma+\lambda g(x_2)^\gamma\right]^{1/\gamma},
\end{split}
\end{equation*}
where $\gamma=s/(1+ms)$. Namely, the marginal function $g(x)$ is $\gamma$-concave.
\end{proof}

\section{Proof of Theorem \ref{theorem: distribution function}}
Similar to the marginal, the CDF of an $s$-concave distribution might not remain in the same class. This is in sharp contrast with the log-concave distributions. The following lemma from \cite{brascamp2002extensions} provides a useful tool for our analysis of CDF, which basically claims that the measure of any $s$-concave distribution is $\gamma$-concave with a closed-form $\gamma$.

\begin{lemma}[\cite{brascamp2002extensions}, Cor 3.4]
\label{lemma: generalized brunn-minkowski inequality}
The density function $f(x)$ is $s$-concave for $s\ge -1/n$ where $x\in\mathbb{R}^n$, if and only if the corresponding probability measure $\mu$ is $\gamma$-concave for $\gamma=\frac{s}{1+ns}$, namely,
$\mu(\lambda A+(1-\lambda)B)\ge \left[\lambda\mu(A)^\gamma+(1-\lambda)\mu(B)^\gamma\right]^{1/\gamma}$,
for any $A,B\subseteq \mathbb{R}^n$ and $\lambda\in[0,1]$, where $\mu(A)=\int_A f(x)dx$.
\end{lemma}

Lemma \ref{lemma: generalized brunn-minkowski inequality} is an extension of celebrated Brunn-Minkowski theorem. The following theorem concerning the CDF of an $s$-concave distribution is a straightforward result from Lemma \ref{lemma: generalized brunn-minkowski inequality}.

\noindent{\textbf{Theorem \ref{theorem: distribution function} (restated)}
\emph{The CDF of $s$-concave distribution in $\mathbb{R}^n$ is $\gamma$-concave, where $\gamma=\frac{s}{1+ns}$ and $s\ge -\frac{1}{n}$.}

\begin{proof}
Denote by $F(x)$ the CDF. Applying Lemma \ref{lemma: generalized brunn-minkowski inequality} to the set $A=\{x: x\le x_1\}$ and $B=\{x: x\le x_2\}$ and taking into account that $F(x_1)=\mu(A)$, $F(x_2)=\mu(B)$, and $F(\lambda x_1+(1-\lambda)x_2)=\mu(\lambda A+(1-\lambda)B)$, we have the result.
\end{proof}

\section{Proof of Theorem \ref{theorem: tail probability}}
Tail probability is one of the most distinct characteristic of $s$-concave distributions compared to the nearly log-concavity.  To study this, we first require a concentration result from \cite{bobkov2007large}.

\begin{lemma}[\cite{bobkov2007large}, Thm 5.2]
\label{lemma: tail probability median}
Let $f$ be a Borel function on $\R^n$ and let $m$ be a median for $|f|$ w.r.t. a $\kappa$-concave measure $\mu$, where $\kappa<0$. Then for every $t>1$ such that $4\delta_f(\frac{1}{t})\le 1$, we have
$
\mu[|f|>mt]\le\left[1-\frac{c\kappa}{\delta_f(\frac{1}{t})}\right]^{1/\kappa},
$
where $c$ is a constant,
$
\delta_f(\epsilon)=\sup_{x,y} \mbox{mes}\{t\in(0,1):|f(tx+(1-t)y)|\le\epsilon|f(x)|\},\ 0\le\epsilon\le 1,
$
and mes stands for the Lebesgue measure.
\end{lemma}

Now we are ready to bound the tail probability of $s$-concave density. While it can be shown that the nearly log-concave distribution has an exponentially small tail~(Theorem 11, \cite{balcan2013active}), the tail of $s$-concave distribution is fat, as proved by the following theorem.

\noindent{\textbf{Theorem \ref{theorem: tail probability} (restated)}
\emph{Let $x$ come from an isotropic distribution over $\R^n$ with an $s$-concave density. Then for every $t\ge 16$, we have
$\Pr[\|x\|>\sqrt{n}t]\le \left[1-\frac{cst}{1+ns}\right]^{(1+ns)/s}$,
where $c$ is an absolute constant.}

\begin{proof}
Set function $f(x)$ in Lemma \ref{lemma: tail probability median} as $\|x\|$. Bobkov~\cite{bobkov2007large} claimed that $\delta_f(\epsilon)\le2\epsilon$. Also, Lemma \ref{lemma: generalized brunn-minkowski inequality} implies that the probability measure is $\kappa=\frac{s}{1+ns}$-concave.

By the definition of the median $m$, the Markov inequality, and the Jensen inequality, we have
$\frac{1}{2}=\Pr[\|x\|\ge m]\le \frac{\E\|x\|}{m}\le\frac{\sqrt{\E\|x\|^2}}{m}=\frac{\sqrt{n}}{m}$,
where the last equality is due to the isotropicity assumption. So by Lemma \ref{lemma: tail probability median}, we have that for every $t\ge 8$,
$\Pr[\|x\|>2\sqrt{n}t]\le \Pr[\|x\|>mt]\le \left[1-cst/(1+ns)\right]^{(1+ns)/s}$.
Replacing $t$ with $t/2$, the proof is completed.
\end{proof}

\section{Proof of Lemma \ref{lemma: halfspace probability}}
\noindent{\textbf{Lemma \ref{lemma: halfspace probability} (restated)}
\emph{Let $X$ be a random variable drawn from a one-dimensional distribution with $s$-concave density for $-1/2\le s\le 0$. Then
\begin{equation*}
\Pr(X\ge \E X)\ge (1+\gamma)^{-1/\gamma},
\end{equation*}
for $\gamma=s/(1+s)$.}

\medskip
\begin{proof}
Without loss of generality, we assume that $\E X=0$ and $|X|\le K$. The general case then follows by translation transformation and approximating a general distribution with $s$-concave density by such bounded distributions.

Let $G(x)=\Pr(X\le x)$ be the CDF of the $s$-concave density. We first prove that $\Pr(X\le \E X)\ge (1+\gamma)^{-1/\gamma}$. By Theorem \ref{theorem: distribution function}, $G(x)$ is $\gamma$-concave, monotone increasing such that $G(x)=0$ for $x\le -K$ and $G(x)=1$ for $x\ge K$, where $-1\le \gamma=\frac{s}{1+s}\le0$. Notice that the assumption of centroid 0 implies that
$
\int_{-K}^KxG'(x)dx=0,
$
which equivalently means
$
\int_{-K}^KG(x)dx=K
$
by integration by parts. Our goal is to prove that $G(0)\ge (1+\gamma)^{-1/\gamma}$.

The function $G^\gamma$ is concave for $\gamma<0$. Thus it lies above its tangent at $0$. This means that $G(x)\le G(0)(1+\gamma cx)^\frac{1}{\gamma}$, where $c=G'(0)/G(0)>0$. We now set $K$ large enough so that $1/c<K$. Then
\begin{equation*}
G(x)\le
\begin{cases}
G(0)(1+\gamma cx)^\frac{1}{\gamma}, & \mbox{if}\ x\le 1/c,\\
1, & \mbox{if}\ x> 1/c.
\end{cases}
\end{equation*}
So
\begin{equation*}
\begin{split}
K&=\int_{-K}^KG(x)dx\\
&\le \int_{-K}^{1/c}G(0)(1+\gamma cx)^\frac{1}{\gamma}dx+\int_{1/c}^K1dx\\
&=\frac{G(0)}{c(\gamma+1)}[(1+\gamma)^{\frac{1}{\gamma}+1}-(1-\gamma cK)^{\frac{1}{\gamma}+1}]+K-\frac{1}{c}\\
&\le\frac{G(0)(1+\gamma)^{\frac{1}{\gamma}}}{c}+K-\frac{1}{c},
\end{split}
\end{equation*}
which implies that $G(0)\ge (1+\gamma)^{-1/\gamma}$ as claimed. Replacing $X$ with $Y=-X$, we obtain the result.
\end{proof}

\section{Proof of Lemma \ref{lemma: bound for 1 dim}}
As a preliminary, we first prove the following lemma concerning the moments of $s$-concave distribution.
\begin{lemma}
\label{lemma: moment property}
Let $g:\R_+\rightarrow\R_+$ be an integrable function. Define
$M_n(g)=\int_0^\infty t^ng(t)dt$,
and suppose it exists.
Then

(a) The sequence $\{M_n(g): n=0,1,...\}$ is log-convex, which means $\log M_n(g)$ is convex w.r.t. variable $n$, or equivalently $M_n(g)M_{n+2}(g)\ge M_{n+1}(g)^2$ for any $n\in N$.

(b) If $g$ is monotone decreasing, then the sequence defined by
\begin{equation*}
M_n'(g)=
\begin{cases}
nM_{n-1}(g), & \mbox{if}\ n>0,\\
g(0), & \mbox{if}\ n=0,
\end{cases}
\end{equation*}
is log-convex.

(c) If $g$ is $s$-concave ($s>-1/(n+1)$), then the sequence $T_n(g)\triangleq M_n(g)/B(-1/s-n-1,n+1)$ is log-concave, which means $\log T_n(g)$ is concave w.r.t. $n$, or equivalently $T_n(g)T_{n+2}(g)\le T_{n+1}(g)^2$ for any $n\in N$.

(d) If $g$ is $s$-concave, then
$g(0)M_1(g)\le M_0(g)^2\frac{1+s}{1+2s}$.
\end{lemma}

\medskip
\begin{proof}
The proofs of Parts (a) and (b) are from \cite{lovasz2007geometry}.

(c) The intuition behind the proof is to choose a baseline $s$-concave function $h$ which is at the ``boundary" between the family of $s$-concave function and that of the non $s$-concave function. We show that $h$ satisfies the equation
\begin{equation}
\label{equ: equality for special function}
T_n(h)T_{n+2}(h)=T_{n+1}^2(h).
\end{equation}
Then by the facts that $h$ is at the ``boundary" and $g$ is any $s$-concave function, we have
\begin{equation}
\label{equ: comparison}
T_{n+1}(h)\le T_{n+1}(g).
\end{equation}
The conclusion follows from \eqref{equ: equality for special function} and \eqref{equ: comparison}, and from our choice of $h$ such that $T_n(h)=T_n(g)$ and $T_{n+2}(h)=T_{n+2}(g)$, by adjusting the slope and intercept of the linear function.

Formally, let $h(t)=\beta(1+\gamma t)^{1/s}$ be the above-mentioned baseline $s$-concave function ($\beta,\gamma>0$) such that
\begin{equation*}
M_n(h)=M_n(g)\ \ \ \ \mbox{and}\ \ \ \ M_{n+2}(h)=M_{n+2}(g)
\end{equation*}
(This holds because there are two parameters $\beta,\gamma$ and two equations). That means
\begin{equation*}
\int_0^\infty t^n(h(t)-g(t))dt=0\ \ \ \ \mbox{and}\ \ \ \ \int_0^\infty t^{n+2}(h(t)-g(t))dt=0.
\end{equation*}
Then it follows that the graph of $h$ must intersect the graph of $g$ at least twice. Since $g$ is $s$-concave, which implies the uni-modality, the graphs of $h$ and $g$ intersect exactly at two points $0\le a<b$. Moreover, $h\le g$ in the interval $[a,b]$ and $h\ge g$ outside the interval. That is to say, $(t-a)(t-b)$ has the same sign as $h-g$. Thus
\begin{equation*}
\int_0^{\infty} (t-a)(t-b)t^n(h(t)-g(t))dt\ge 0.
\end{equation*}
Namely,
\begin{equation*}
0=\int_0^{\infty}t^{n+2}(h(t)-g(t))dt+ab\int_0^{\infty}t^n(h(t)-g(t))dt\ge (a+b)\int_0^\infty t^{n+1}(h(t)-g(t))dt.
\end{equation*}
This implies that
\begin{equation*}
M_{n+1}(h)=\int_0^{\infty}t^{n+1}h(t)dt\le\int_0^\infty t^{n+1}g(t)dt=M_{n+1}(g).
\end{equation*}
Since
\begin{equation*}
\label{equ: integral}
M_n(h)=\int_0^\infty t^n\beta(1+\gamma t)^{1/s}dt=B(-1/s-n-1,n+1)\frac{\beta}{\gamma^{n+1}}
\end{equation*}
for $s>-1/(n+1)$, we have
\begin{equation*}
\begin{split}
\frac{M_n(g)}{B(-1/s-n-1,n+1)}\frac{M_{n+2}(g)}{B(-1/s-n-3,n+3)}&=\frac{M_n(h)}{B(-1/s-n-1,n+1)}\frac{M_{n+2}(h)}{B(-1/s-n-3,n+3)}\\
&=\frac{\beta}{\gamma^{n+1}}\cdot \frac{\beta}{\gamma^{n+3}}\\
&=\left(\frac{M_{n+1}(h)}{B(-1/s-n-2,n+2)}\right)^2\\
&\le \left(\frac{M_{n+1}(g)}{B(-1/s-n-2,n+2)}\right)^2,
\end{split}
\end{equation*}
as desired.

(d) The proof is almost the same as that of Part (c). Let $h(t)=\beta(1+\gamma t)^{1/s}$ be an $s$-concave function ($\beta,\gamma>0$) such that
\begin{equation*}
h(0)=g(0)\ \ \ \ \mbox{and}\ \ \ \ M_1(h)=M_1(g).
\end{equation*}
So the graphs of $h$ and $g$ intersect exactly at two points $0$ and $a>0$, and hence
\begin{equation*}
\int_0^{\infty}t(t-a)t^{-1}(h(t)-g(t))dt\ge 0.
\end{equation*}
That means
\begin{equation*}
0=\int_0^{\infty}t(h(t)-g(t))dt\ge a\int_0^{\infty}(h(t)-g(t))dt,
\end{equation*}
or equivalently,
\begin{equation*}
M_0(h)\le M_0(g).
\end{equation*}
Note that $h(0)M_1(h)=M_0(h)^2\frac{1+s}{1+2s}$ by \eqref{equ: integral}.
Then the conclusion follows by the fact
\begin{equation*}
g(0)M_1(g)=h(0)M_1(h)=M_0(h)^2\frac{1+s}{1+2s}\le M_0(g)^2\frac{1+s}{1+2s}.
\end{equation*}
\end{proof}

Now we are ready to prove Lemma \ref{lemma: bound for 1 dim}.

\medskip
\noindent{\textbf{Lemma \ref{lemma: bound for 1 dim}} (restated)}
\emph{Let $g:\mathbb{R}\rightarrow\mathbb{R}_+$ be an isotropic $s$-concave density function and $s> -1/3$.}

\emph{
(a) For all $x$, $g(x)\le \frac{1+s}{1+3s}$.}

\emph{
(b) We have $g(0)\ge\sqrt{\frac{1}{3(1+\gamma)^{3/\gamma}}}$, where $\gamma=\frac{s}{s+1}$.}
\medskip
\begin{proof}
(a) Let $z$ be the maximum point of function $g$. Intuitively, if the value of the function $g$ evaluated at $z$ is too large, the corresponding distribution has a small deviation from $z$ (Second part of the proof below). However, the moment property of Lemma \ref{lemma: moment property} restricts that the second moment cannot be too small (First part of the proof below), which leads to a contradiction.

Formally, suppose that $g(z)>\frac{1+s}{1+3s}$. Define
\begin{equation*}
M_i=\int_z^{\infty}(x-z)^ig(x)dx,\ \ \ \ \mbox{and}\ \ \ \ N_i=\int_{-\infty}^z(z-x)^ig(x)dx.
\end{equation*}
By the isotropicity of function $g$, we have
\begin{equation*}
M_0+N_0=1,\ \ \ \ N_1-M_1=z,\ \ \ \ M_2+N_2=1+z^2.
\end{equation*}
Thus
\begin{equation*}
\begin{split}
M_2+N_2&=(M_0+N_0)^2+(M_1-N_1)^2\\
&=(M_0-M_1)^2+(N_0-N_1)^2+2(M_0N_0-M_1N_1)+2(M_0M_1+N_0N_1)\\
&\ge 2(M_0M_1+N_0N_1),
\end{split}
\end{equation*}
where the last inequality holds since, by Lemma \ref{lemma: moment property} (d), we have $M_1\le\frac{M_0^2}{g(z)}\frac{1+s}{1+2s}\le M_0^2\le M_0$ and $N_1\le\frac{N_0^2}{g(z)}\frac{1+s}{1+2s}\le N_0^2\le N_0$.

On the other hand, by Lemma \ref{lemma: moment property} (c) (d),
\begin{equation*}
M_2\le \frac{2M_1^2}{M_0}\frac{1+2s}{1+3s}\le \frac{2M_1M_0}{g(z)}\frac{1+s}{1+3s}<2M_1M_0,
\end{equation*}
and similarly, $N_2<2N_1N_0$. That means
\begin{equation*}
M_2+N_2<2(M_0M_1+N_0N_1),
\end{equation*}
and we obtain a contradiction.

(b) The proof is by Lemma \ref{lemma: moment property} (b) which lower bounds $g(0)$ by the second order moment of $g$, which is $1$ according to isotropicity.

Specifically, without lose of generality, assume that $g(x)$ is monotone decreasing for $x\ge 0$ (otherwise consider $g(-x)$, since function $g$ is uni-modal). Define $g_0$ as the restriction of $g$ to the non-negative semi-line. Then by Lemma \ref{lemma: moment property} (b), we have
\begin{equation*}
M_1'(g_0)^3\le M_0'(g_0)^2M_3'(g_0),
\end{equation*}
which implies
\begin{equation*}
g(0)\ge \sqrt{\frac{M_0(g_0)^3}{3M_2(g_0)}}.
\end{equation*}
Note that $M_2(g_0)\le M_2(g)=1$, and by Lemma \ref{lemma: halfspace probability},
\begin{equation*}
M_0(g_0)=\int_0^\infty g(t)dt=\Pr[X\ge \E X]\ge (1+\gamma)^{-1/\gamma}.
\end{equation*}
Thus we have
\begin{equation*}
g(0)\ge \sqrt{\frac{1}{3(1+\gamma)^{3/\gamma}}},
\end{equation*}
where $\gamma=s/(1+s)$.
\end{proof}

\section{Proof of Theorem \ref{theorem: upper and lower bounds}}
\label{section: proof of upper and lower bound}

\noindent{\textbf{Theorem \ref{theorem: upper and lower bounds}} (restated)}
\emph{Let $f:\R^n\rightarrow\R_+$ be an isotropic $s$-concave density.}

\emph{
(a) Let $d=(1+\gamma)^{-1/\gamma}\frac{1+3\beta}{3+3\beta}$, where $\beta=\frac{s}{1+(n-1)s}$ and $\gamma=\frac{\beta}{1+\beta}$. For any $u\in\R^n$ such that $\|u\|\le d$, we have \red{$f(u)\ge \left(\frac{\|u\|}{d}((2-2^{-(n+1)s})^{-1}-1)+1\right)^{1/s}f(0)$}.}

\emph{
(b) $f(x)\le f(0)\left[\left(\frac{1+\beta}{1+3\beta}\sqrt{3(1+\gamma)^{3/\gamma}}2^{n-1+1/s}\right)^{s}-1\right]^{1/s}$ for every $x$ ($s\ge -\frac{1}{2n+3}$).}

\emph{
(c) There exists an $x\in\R^n$ such that $f(x)>(4e\pi)^{-n/2}$.}

\emph{
(d) We have $(4e\pi)^{-n/2}\left[\left(\frac{1+\beta}{1+3\beta}\sqrt{3(1+\gamma)^{3/\gamma}}2^{n-1+1/s}\right)^{s}-1\right]^{-1/s}<f(0)\le \red{(2-2^{-(n+1)s})^{1/s}\frac{n\Gamma(n/2)}{2\pi^{n/2}d^n}}$.}

\emph{
(e) \red{$f(x)\le (2-2^{-(n+1)s})^{1/s}\frac{n\Gamma(n/2)}{2\pi^{n/2}d^n}\left[\left(\frac{1+\beta}{1+3\beta}\sqrt{3(1+\gamma)^{3/\gamma}}2^{n-1+1/s}\right)^{s}-1\right]^{1/s}$} for every $x$.}

\emph{
(f) For any line $\ell$ through the origin, $\int_{\ell} f \le \red{(2-2^{-ns})^{1/s}\frac{(n-1)\Gamma((n-1)/2)}{2\pi^{(n-1)/2}d^{n-1}}}$.}

\medskip
\begin{proof}
(a) 
Formally, suppose that the conclusion does not hold true, i.e., there is a point $u$ such that $\|u\|=t\le d$ and \red{$f(u)<\left(\frac{t}{d}((2-2^{-(n+1)s})^{-1}-1)+1\right)^{1/s}f(0)$}. Define $v=\frac{d}{t}u$ and note that $0\le \frac{t}{d}\le 1$. Therefore, by the $s$-concavity of $f$, we have
\begin{equation*}
f(u)=f\left(\frac{t}{d}v+\left(1-\frac{t}{d}\right)0\right)\ge \left[\frac{t}{d}f(v)^s+\left(1-\frac{t}{d}\right)f(0)^s\right]^{1/s},
\end{equation*}
which together with $f(u)<\left(\frac{t}{d}((2-2^{-(n+1)s})^{-1}-1)+1\right)^{1/s}f(0)$ implies \red{$f(v)<(2-2^{-(n+1)s})^{-1/s}f(0)$}. Let $H$ be a hyperplane supporting the convex set $\{x\in\R^n: f(x)\ge f(v)\}$ through the point $v$ (the convexity follows from the $s$-concavity of $f$). Define an orthogonal coordinate system in which the hyperplane is parallel to coordinate plane so that it can be represented as $x_1=a$ for some $0<a\le d$. Thus \red{$f(x)<(2-2^{-(n+1)s})^{-1/s}f(0)$} for any $x$ such that $x_1\ge a$. We will prove that this implies that the 1-dimensional marginal is not flat.

Denote by $g$ the first marginal of the $n$-dimensional function $f$. Then $g$ is isotropic and $\beta=\frac{s}{1+(n-1)s}$-concave by Theorem \ref{theorem: marginal}, and $g(x)\le \frac{1+\beta}{1+3\beta}$ for all $x$ by Lemma \ref{lemma: bound for 1 dim} (a). We prove that
\begin{equation*}
g(2b)<\frac{g(b)}{4}
\end{equation*}
for any $b\ge a$, which means that the 1-dimensional function is not flat.
To see this, by the $s$-concavity of function $f$, we have that, for every $x$ such that $x_1\ge a$,
\begin{equation*}
f(2x)^s\ge 2f(x)^s-f(0)^s>2^{-(n+1)s}f(x)^s.
\end{equation*}
Namely, $f(2x)< 2^{-(n+1)}f(x)$. Hence
\begin{equation*}
g(2b)=\int_{(x_1=2b)}f(x)dx_2...dx_n< 2^{-(n+1)}2^{n-1}\int_{(x_1=b)}f(x)dx_2...dx_n=\frac{g(b)}{4}.
\end{equation*}
So
\begin{equation*}
\int_a^\infty g(y)dy=\int_a^{2a} g(y)dy+\int_{2a}^\infty g(y)dy< \int_a^{2a} g(y)dy+\frac{1}{2}\int_a^{\infty} g(y)dy.
\end{equation*}
Namely, by Lemma \ref{lemma: bound for 1 dim} (a),
\begin{equation*}
\int_a^{\infty} g(y)dy< 2\int_a^{2a} g(y)dy\le 2a\frac{1+\beta}{1+3\beta}.
\end{equation*}
So
\begin{equation*}
\int_0^{\infty} g(y)dy=\int_0^a g(y)dy+\int_a^{\infty} g(y)dy< 3a\frac{1+\beta}{1+3\beta}\le 3d\frac{1+\beta}{1+3\beta}=(1+\gamma)^{-1/\gamma},
\end{equation*}
which leads to a contradiction with Lemma \ref{lemma: halfspace probability}.

(b) If $f(w)\le f(0)$ for every $w$, then the conclusion holds true. Otherwise, let $w$ be the point such that $f(w)> f(0)$. Let $H_0$ be the hyperplane through $0$ which supports the convex set $\{x\in\R^n:f(x)\ge f(0)\}$. By defining an orthogonal system, we may set $H_0$ as the hyperplane $x_1=0$, and so $f(x)\le f(0)$ for any $x$ such that $x_1=0$. Define $g$, which is a $\beta=\frac{s}{1+(n-1)s}$-concave function, as the first marginal of function $f$. Denote by $H_t$ the hyperplane $x_1=t$. Without loss of generality, we assume that $w\in H_b$ with $b>0$.

Let $x$ be any point on $H_0$ and $x'$ be the intersection between line segment $[x,w]$ and $H_{b/2}$. Then by the $s$-concavity of $f$ and $f(x)\le f(0)$ for $x\in H_0$, we have
\begin{equation*}
f(x')\ge \left[\frac{1}{2}f(x)^s+\frac{1}{2}f(w)^s\right]^{1/s}\ge\left(\frac{1}{2}\right)^{1/s}f(x)\left[1+\left(\frac{f(w)}{f(0)}\right)^{s}\right]^{1/s}.
\end{equation*}
Thus
\begin{equation*}
g(b/2)=\int_{(x_1=b/2)}f(x)dx_2...dx_n\ge\frac{1}{2^{n-1+1/s}}\left[1+\left(\frac{f(w)}{f(0)}\right)^{s}\right]^{1/s}g(0).
\end{equation*}
By Lemma \ref{lemma: bound for 1 dim} (a) (b), we have
\begin{equation*}
\frac{1+\beta}{1+3\beta}\ge g(b/2)\ge \frac{1}{2^{n-1+1/s}}\left[1+\left(\frac{f(w)}{f(0)}\right)^{s}\right]^{1/s}\sqrt{\frac{1}{3(1+\gamma)^{3/\gamma}}},
\end{equation*}
where $\gamma=\beta/(1+\beta)$. Note that $s\ge -\frac{1}{2n+3}$ implies $\frac{1+\beta}{1+3\beta}2^{n-1+1/s}\sqrt{3(1+\gamma)^{3/\gamma}}<1$. So
\begin{equation*}
f(w)\le f(0)\left[\left(\frac{1+\beta}{1+3\beta}\sqrt{3(1+\gamma)^{3/\gamma}}2^{n-1+1/s}\right)^{s}-1\right]^{1/s}.
\end{equation*}

(c) The proof of Part (c) follows from \cite{lovasz2007geometry}.

(d) The proof of lower bound follows from Parts (b) and (c).

For the upper bound, by Part (a), we have
\begin{equation*}
1=\int_{\R^n}f(x)dx\ge \int_{\|x\|\le d}f(x)dx\ge \red{d^n \text{vol}(B_{n-1})(2-2^{-(n+1)s})^{-1/s}f(0),}
\end{equation*}
where $\text{vol}(B_{n-1})$ represents the volume of $n-1$-dimensional unit ball. So
\begin{equation*}
f(0)\le \red{\frac{(2-2^{-(n+1)s})^{1/s}}{d^n\text{vol}(B_{n-1})}=(2-2^{-(n+1)s})^{1/s}\frac{n\Gamma(n/2)}{2\pi^{n/2}d^n}.}
\end{equation*}

(e) The proof of (e) follows from Parts (b) and (d).

(f) Define an orthogonal coordinate system in which $\ell$ is the $x_n$-axis. Let $h$ be the marginal of function $f$ over first $n-1$ variables, namely,
\begin{equation*}
h(x_1,...,x_{n-1})=\int f(x_1,...,x_{n-1},x_n)dx_n.
\end{equation*}
Then
\begin{equation*}
\int_{\ell} f=h(0)\le \red{(2-2^{-ns})^{1/s}\frac{(n-1)\Gamma((n-1)/2)}{2\pi^{(n-1)/2}d^{n-1}}.}
\end{equation*}
\end{proof}

\section{Proof of Lemma \ref{lemma: more refined upper bound}}
\label{section: proof of more refined upper bound}

\noindent{\textbf{Lemma \ref{lemma: more refined upper bound}} (restated)}
\emph{Let $f:\R^n\rightarrow \R_+$ be an isotropic $s$-concave density. Then $f(x)\le \beta_1(n,s)(1-s\beta_2(n,s)\|x\|)^{1/s}$ for every $x\in\R^n$, where
\begin{equation*}
\red{\beta_1(n,s)=(2-2^{-(n+1)s})^{\frac{1}{s}}\frac{1}{2\pi^{n/2}d^n}(1\hspace{-0.1cm}-\hspace{-0.1cm}s)^{-\frac{1}{s}}n\Gamma(n/2)\hspace{-0.1cm}\left[\left(\frac{1+\beta}{1+3\beta}\sqrt{3(1+\gamma)^{3/\gamma}}2^{n-1+1/s}\right)^{s}\hspace{-0.2cm}-\hspace{-0.1cm}1\right]^{1/s},}
\end{equation*}
and
\begin{equation*}
\red{\beta_2(n,s)=\frac{2\pi^{(n-1)/2}d^{n-1}}{(n-1)\Gamma((n-1)/2)}(2-2^{-ns})^{-1/s}\frac{[(a+(1-s)\beta_1(n,s)^s)^{1+1/s}-a^{1+1/s}]s}{\beta_1(n,s)^s(1+s)(1-s)},}
\end{equation*}
$d=(1+\gamma)^{-\frac{1}{\gamma}}\frac{1+3\beta}{3+3\beta}$, $\beta=\frac{s}{1+(n-1)s}$, $\gamma=\frac{\beta}{1+\beta}$, $a\hspace{-0.1cm}=\hspace{-0.1cm}(4e\pi)^{-\frac{ns}{2}}\hspace{-0.1cm}\left[\left(\frac{1+\beta}{1+3\beta}\sqrt{3(1+\gamma)^{3/\gamma}}2^{n-1+1/s}\right)^{s}\hspace{-0.2cm}-\hspace{-0.1cm}1\right]^{-1}$.}

\medskip
\begin{proof}
We first note that when $\|x\|\le 1/\beta_2$, $\beta_1(1-s\beta_2\|x\|)^{1/s}\ge \beta_1(1-s)^{1/s}\ge f(x)$ by Theorem \ref{theorem: upper and lower bounds} (e). So the conclusion holds.

We now assume that there is a point $v$ such that $\|v\|>1/\beta_2$ but $f(v)>\beta_1(1-s\beta_2\|v\|)^{1/s}$. Denote by $[0,v]$ the line segment between the origin $0$ and the point $v$, and let $\ell$ be the line through $v$ and $0$. We will prove that
\begin{equation*}
\int_\ell f > \red{(2-2^{-ns})^{1/s}\frac{(n-1)\Gamma((n-1)/2)}{2\pi^{(n-1)/2}d^{n-1}},}
\end{equation*}
which leads to a contradiction with Theorem \ref{theorem: upper and lower bounds} (f). Let $x$ be the convex combination of points $0$ and $v$, i.e., $x=(1-\|x\|/\|v\|)0+(\|x\|/\|v\|)v$, where $0\le \|x\|\le\|v\|$. Then by the $s$-concavity of $f$ and Theorem \ref{theorem: upper and lower bounds} (d),
\begin{equation*}
\begin{split}
f(x)&\ge \left[\left(1-\frac{\|x\|}{\|v\|}\right)f(0)^s+\frac{\|x\|}{\|v\|}f(v)^s\right]^{1/s}\\&> \red{\left[f(0)^s+\frac{\|x\|}{\|v\|}f(v)^s\right]^{1/s}}\\
&> \left[f(0)^s+\frac{\|x\|}{\|v\|}\beta_1^s-s\beta_1^s\beta_2\|x\|\right]^{1/s}\\&>\left[f(0)^s+(1-s)\beta_1^s\beta_2\|x\|\right]^{1/s}\\
&> \left\{(4e\pi)^{-ns/2}\left[\left(\frac{1+\beta}{1+3\beta}\sqrt{3(1+\gamma)^{3/\gamma}}2^{n-1+1/s}\right)^{s}-1\right]^{-1}+(1-s)\beta_1^s\beta_2\red{\|x\|}\right\}^{1/s}.
\end{split}
\end{equation*}
Thus
\begin{equation*}
\begin{split}
\int_\ell f\ge \int_{[0,v]}f&=\int_0^{\|v\|} \hspace{-0.2cm}\left\{(4e\pi)^{-ns/2}\left[\left(\frac{1+\beta}{1+3\beta}\sqrt{3(1+\gamma)^{3/\gamma}}2^{n-1+1/s}\right)^{s}\hspace{-0.2cm}-\hspace{-0.1cm}1\right]^{-1}\hspace{-0.3cm}+\hspace{-0.1cm}(1-s)\beta_1^s\beta_2r\right\}^{1/s}\hspace{-0.2cm} dr\\
&>\int_0^{1/\beta_2} \hspace{-0.2cm}\left\{(4e\pi)^{-ns/2}\left[\left(\frac{1+\beta}{1+3\beta}\sqrt{3(1+\gamma)^{3/\gamma}}2^{n-1+1/s}\right)^{s}\hspace{-0.2cm}-\hspace{-0.1cm}1\right]^{-1}\hspace{-0.3cm}+\hspace{-0.1cm}(1-s)\beta_1^s\beta_2r\right\}^{1/s}\hspace{-0.2cm} dr\\
&\ge \red{(2-2^{-ns})^{1/s}\frac{(n-1)\Gamma((n-1)/2)}{2\pi^{(n-1)/2}d^{n-1}}.}
\end{split}
\end{equation*}
So the proof is completed.
\end{proof}

\section{Proof of Theorem \ref{theorem: probability within margin}}
\noindent{\textbf{Theorem \ref{theorem: probability within margin}} (restated)}
\emph{Let $\mathcal{D}$ be an isotropic $s$-concave distribution in $\R^n$. Denote by $f_3(s,n)=2(1+ns)/(1+(n+2)s)$. Then for any unit vector $w$,
\begin{equation}
\Pr_{x\sim\mathcal{D}}[|w\cdot x|\le t]\le f_3(s,n)t.
\end{equation}
Moreover, if $t\le d=\left(\frac{1+2\gamma}{1+\gamma}\right)^{-\frac{1+\gamma}{\gamma}}\frac{1+3\gamma}{3+3\gamma}$ where $\gamma=\frac{s}{1+(n-1)s}$, then
\begin{equation}
\Pr_{x\sim\mathcal{D}}[|w\cdot x|\le t]> f_2(s,n)t,
\end{equation}
where $f_2(s,n)=2(2-2^{-2\gamma})^{-1/\gamma}(4e\pi)^{-1/2}\left(2\left(\frac{1+\gamma}{1+3\gamma}\sqrt{3}\left(\frac{1+2\gamma}{1+\gamma}\right)^{\frac{3+3\gamma}{2\gamma}}\right)^\gamma-1\right)^{-1/\gamma}$.}

\begin{proof}
Define an orthogonal coordinate system in which $w$ is an axis. Then the distribution of $w\cdot x$ is equal to the first marginal of the distribution $\mathcal{D}$, with isotropic $\gamma=\frac{s}{1+(n-1)s}$-concave density $g$ by Theorem \ref{theorem: marginal}. According to the upper bound given by Lemma \ref{lemma: bound for 1 dim} (a),
\begin{equation*}
\Pr_{x\sim\mathcal{D}}[|w\cdot x|\le t]=\int_{-t}^t g(x)dx\le \frac{1+\gamma}{1+3\gamma}\int_{-t}^t dx=2\frac{1+ns}{1+(n+2)s}t.
\end{equation*}

We now prove the later part of the theorem by a similar argument. By Theorem \ref{theorem: upper and lower bounds} (a) (d), for 1-dimensional $\gamma$-concave density $f(u)$ and $\|u\|\le d$, we have
\begin{equation*}
\begin{split}
f(u)&\ge (2-2^{-2\gamma})^{-1/\gamma}f(0)\\
&>(2-2^{-2\gamma})^{-1/\gamma}(4e\pi)^{-1/2}\left(2\left(\frac{1+\gamma}{1+3\gamma}\sqrt{3}\left(\frac{1+2\gamma}{1+\gamma}\right)^{\frac{3+3\gamma}{2\gamma}}\right)^\gamma-1\right)^{-1/\gamma}\\
&\triangleq \frac{f_2(s,n)}{2}.
\end{split}
\end{equation*}
Therefore,
\begin{equation*}
\Pr_{x\sim\mathcal{D}}[|w\cdot x|\le t]=\int_{-t}^t g(x)dx> \frac{f_2(s,n)}{2}\int_{-t}^t dx=f_2(s,n)t.
\end{equation*}
\end{proof}

\section{Proof of Theorem \ref{theorem: disagreement and angle}}

\noindent{\textbf{Theorem \ref{theorem: disagreement and angle}} (restated)}
\emph{Assume $\mathcal{D}$ is an isotropic $s$-concave distribution in $\R^n$. Then for any two unit vectors $u$ and $v$ in $\R^n$, we have $d_\mathcal{D}(u,v)=\Pr_{x\sim\mathcal{D}}[\sign(u\cdot x)\not=\sign(v\cdot x)]\ge f_1(s,n)\theta(u,v)$, where $f_1(s,n)=c\red{(2-2^{-3\alpha})^{-\frac{1}{\alpha}}\hspace{-0.1cm}\left[\left(\frac{1+\beta}{1+3\beta}\sqrt{3(1+\gamma)^{3/\gamma}}2^{1+1/\alpha}\right)^{\alpha}\hspace{-0.2cm}-1\right]^{-\frac{1}{\alpha}}}(1+\gamma)^{-2/\gamma}\left(\frac{1+3\beta}{3+3\beta}\right)^2$, $c$ is an absolute constant, $\alpha=\frac{s}{1+(n-2)s}$, $\beta=\frac{s}{1+(n-1)s}$, $\gamma=\frac{s}{1+ns}$.}

\begin{proof}
Consider the 2-dimensional space spanned by vectors $u$ and $v$, and let $\mathcal{D}_2$ be the marginal of distribution $\mathcal{D}$ over that space. Since $d_{\mathcal{D}}(u,v)=d_{\mathcal{D}_2}(u',v')$, where $u'$ and $v'$ are projection of $u$ and $v$, respectively, we only need to focus on the marginal distribution $\mathcal{D}_2$, which has an $\alpha$-concave density according to Theorem \ref{theorem: marginal}, and is isotropic according to Theorem \ref{theorem: marginal}.

Let $A$ be the disagreement region of $u$ and $v$ intersected with the ball of radius $d=(1+\gamma)^{-1/\gamma}\frac{1+3\beta}{3+3\beta}$ in $\R^2$. By Theorem \ref{theorem: upper and lower bounds} (a) and (d),
$
d_{\mathcal{D}}(u,v)=d_{\mathcal{D}_2}(u',v')\ge \mbox{vol}(A)\inf_{x\in A} p(x)\ge f_1(s,n)\theta(u',v')=f_1(s,n)\theta(u,v),
$
where $p(x)$ is the density of distribution $\mathcal{D}_2$.
\end{proof}

\section{Proof of Theorem \ref{theorem: disagreement outside band}}
\label{section: proof of disagreement outside band}

\noindent{\textbf{Theorem \ref{theorem: disagreement outside band}} (restated)}
\emph{Let $u$ and $v$ be two vectors in $\R^n$ and assume that $\theta(u,v)=\theta<\pi/2$. Let $\mathcal{D}$ be an isotropic $s$-concave distribution. Then for any absolute constant $c_1>0$ and any function $f_1(s,n)>0$, there exists a function $f_4(s,n)>0$ such that
\begin{equation*}
\Pr_{x\sim\mathcal{D}}[\mathrm{sign}(u\cdot x)\not=\mathrm{sign}(v\cdot x)\ \text{and}\ |v\cdot x|\ge f_4(s,n)\theta]\le c_1f_1(s,n)\theta,
\end{equation*}
where $f_4(s,n)=\frac{4\beta_1(2,\alpha)B(-1/\alpha-3,3)}{-c_1f_1(s,n)\alpha^3\beta_2(2,\alpha)^3}$, $B(\cdot,\cdot)$ is the beta function, $\alpha=s/(1+(n-2)s)$, $\beta_1(2,\alpha)$ and $\beta_2(2,\alpha)$ are given by Lemma \ref{lemma: more refined upper bound}.}

\medskip
\begin{proof}
Let $E$ be the event that we want to bound. Theorem \ref{theorem: marginal} implies that, without loss of generality, we can focus on the case when $n=2$. Then the projected distribution $\mathcal{D}_2$ has an $\alpha$-concave density, where $\alpha=\frac{s}{1+(n-2)s}$.

We first claim that each member $x$ of $E$ satisfies $\|x\|\ge f_4$. To see this, without loss of generality, we assume that $v\cdot x$ is positive. Then for any $x\in E$, $u\cdot x<0$, i.e., $\theta(u,x)\ge \pi/2$. Hence $\theta(x,v)\ge \theta(u,x)-\theta(u,v)\ge \pi/2-\theta$. Since $|v\cdot x|\ge f_4\theta$ implies $\|x\|\cos\theta(v,x)\ge f_4\theta$, we must have $\|x\|\cos(\pi/2-\theta)\ge f_4\theta$, namely, $\|x\|\ge f_4\theta/\sin(\theta)\ge f_4$. Let $\ball(r)$ denote the ball of radius $r$ centered at the origin. This implies that
\begin{equation*}
\Pr[E]=\sum_{i=1}^\infty \Pr[E\cap (\ball((i+1)f_4)-\ball(if_4))].
\end{equation*}

Denote by $f(x_1,x_2)$ the $\alpha$-concave density function of $\mathcal{D}_2$. For any term $i\ge 1$, by Lemma \ref{lemma: more refined upper bound}, we have
\begin{equation*}
\begin{split}
&\ \ \ \ \Pr[E\cap (\ball((i+1)f_4)-\ball(if_4))]\\
&=\int_{(x_1,x_2)\in \ball((i+1)f_4)-\ball(if_4)}1_E(x_1,x_2)f(x_1,x_2)dx_1dx_2\\
&\le \int_{(x_1,x_2)\in \ball((i+1)f_4)-\ball(if_4)}1_E(x_1,x_2)\beta_1(2,\alpha)(1-\alpha\beta_2(2,\alpha)\|x\|)^{1/\alpha}dx_1dx_2\\
&\le \beta_1(2,\alpha)\left(1-\alpha\beta_2(2,\alpha)if_4\right)^{1/\alpha}\int_{(x_1,x_2)\in \ball((i+1)f_4)-\ball(if_4)}1_E(x_1,x_2)dx_1dx_2\\
&\le \beta_1(2,\alpha)\left(1-\alpha\beta_2(2,\alpha)if_4\right)^{1/\alpha}\int_{(x_1,x_2)\in \ball((i+1)f_4)}1_E(x_1,x_2)dx_1dx_2.
\end{split}
\end{equation*}
Denote by $B_1$ the unit ball in $\R^2$. Notice that
\begin{equation*}
\int_{(x_1,x_2)\in \ball((i+1)f_4)}1_E(x_1,x_2)dx_1dx_2\le \text{vol}(B_1)(i+1)^2f_4^2\theta/\pi.
\end{equation*}
Thus
\begin{equation*}
\begin{split}
\Pr[E]&\le\sum_{i=1}^\infty \beta_1(2,\alpha)\left(1-\alpha\beta_2(2,\alpha)if_4\right)^{1/\alpha}\text{vol}(B_1)(i+1)^2f_4^2\theta/\pi\\
&\le \frac{4f_4^2}{\pi}\text{vol}(B_1)\beta_1(2,\alpha)\theta\sum_{i=1}^\infty (1-\alpha\beta_2(2,\alpha)if_4)^{1/\alpha}i^2\\
&\le \frac{4f_4^2}{\pi}\text{vol}(B_1)\beta_1(2,\alpha)\theta\int_0^\infty x^2(1-\alpha\beta_2(2,\alpha)f_4x)^{1/\alpha}dx\\
&=\frac{4f_4^2}{\pi}\text{vol}(B_1)\beta_1(2,\alpha) \frac{B(-1/\alpha-3,3)}{(-\alpha\beta_2(2,\alpha)f_4)^3}\times \theta\\
&=4\beta_1(2,\alpha) \frac{B(-1/\alpha-3,3)}{-\alpha^3\beta_2(2,\alpha)^3f_4}\times \theta.
\end{split}
\end{equation*}
Choosing $f_4(s,n)=\frac{4\beta_1(2,\alpha)B(-1/\alpha-3,3)}{-c_1f_1(s,n)\alpha^3\beta_2(2,\alpha)^3}$, the proof is completed.
\end{proof}

\section{Proof of Theorem \ref{theorem: 1-D variance}}
\label{section: proof of 1D variance}

Before proceeding, we first prove the following lemma which is critical to the proof of Theorem \ref{theorem: 1-D variance}.
\begin{lemma}
\label{lemma: probability of outside of margin}
For $d$ given by Theorem \ref{theorem: upper and lower bounds} (a), there exist such that for any isotropic $s$-concave distribution $\cD$, for any $a$ such that $\|a\|\le 1$ and $\|u-a\|\le r$, for any $0<t\le d$, and for any $K\ge 4$, we have
\begin{equation*}
\Pr_{X\sim \cD_{u,t}}(|a\cdot x|> K\sqrt{r^2+t^2})\le \frac{4\beta_1(2,\eta)}{f_2(s,n)\beta_2(2,\eta)}\frac{1}{\eta+1}\left(1-c\eta\beta_2(2,\eta)K\sqrt{1+\left(\frac{t}{r}\right)^2}\right)^{\frac{\eta+1}{\eta}},
\end{equation*}
where $\beta_1(2,\eta)$, $\beta_2(2,\eta)$, and $Q(\gamma)$, are given by Lemma \ref{lemma: more refined upper bound} and Theorem \ref{theorem: probability within margin}, respectively, $\eta=\frac{s}{1+(n-2)s}$, and $c$ is an absolute constant.
\end{lemma}

\medskip
\begin{proof}
Without loss of generality, we assume that $u=(1,0,...,0)$. Let $a'=(0,a_2,...,a_d)$ and $X=(x_1,x_2,...,x_d)\sim \cD_{u,t}$. So the probability that we want to bound is
\begin{equation*}
\Pr_{X\sim \cD_{u,t}}(|a\cdot x|> K\sqrt{r^2+t^2})=\frac{\Pr_{x\sim \cD}(|a\cdot x|> K\sqrt{r^2+t^2},\ |x_1|\le t)}{\Pr_{x\sim \cD}(|x_1|\le t)}.
\end{equation*}
According to Theorem \ref{theorem: probability within margin}, there is a function $Q(\gamma)$ such that the denominator obeys the following lower bound when $t\le d$:
\begin{equation*}
\Pr_{X\sim \cD}(|x_1|\le t)\ge f_2(s,n)t.
\end{equation*}
So the remainder of the proof is to bound the numerator. Note that we have
\begin{equation*}
\begin{split}
\Pr_{x\sim \cD}(|a\cdot x|> K\sqrt{r^2+t^2},\ |x_1|\le t)&\le \Pr_{x\sim \cD}(|a'\cdot x|> K\sqrt{r^2+t^2}-t,\ |x_1|\le t)\\
&\le \Pr_{x\sim \cD}(|a'\cdot x|> (K-1)\sqrt{r^2+t^2},\ |x_1|\le t).
\end{split}
\end{equation*}
Denote by $a''=\frac{a'}{\|a'\|}$. Define random variable $Y$ as $a''\cdot x$ and $Z$ as $x_1$ where $x\sim\cD$. Then the joint distribution of $Y$ and $Z$ is isotropic $\beta$-concave with $\eta=\frac{s}{1+(n-2)s}$. Let $f(y,z)$ be the density of such a distribution. Then we can bound the numerator by
\begin{equation*}
\begin{split}
4\Pr_{x\sim \cD}(a'\cdot x>(K-1)\sqrt{r^2+t^2},\ 0\le x_1\le t)&=4\Pr_{X\sim \cD}(a''\cdot x>\frac{(K-1)\sqrt{r^2+t^2}}{\|a'\|},\ 0\le x_1\le t)\\
&\le 4\int_0^t\int_{\frac{(K-1)\sqrt{r^2+t^2}}{\|a'\|}}^\infty f(y,z) dydz.
\end{split}
\end{equation*}
By Lemma \ref{lemma: more refined upper bound}, we note that
\begin{equation*}
f(y,z)\le \beta_1(2,\eta)(1-\eta\beta_2(2,\eta)\sqrt{y^2+z^2})^{1/\eta}.
\end{equation*}
Therefore, the numerator can be upper bounded by
\begin{equation}
\begin{split}
&\quad \ 4\beta_1(2,\eta)\int_0^t\int_{\frac{(K-1)\sqrt{r^2+t^2}}{\|a'\|}}^\infty (1-\eta\beta_2(2,\eta)\sqrt{y^2+z^2})^{1/\eta} dydz\\
&\le 4\beta_1(2,\eta)\int_0^t\int_{\frac{(K-1)\sqrt{r^2+t^2}}{\|a'\|}}^\infty (1-\eta\beta_2(2,\eta)y)^{1/\eta} dydz\\
&=4t\beta_1(2,\eta)\int_{\frac{(K-1)\sqrt{r^2+t^2}}{\|a'\|}}^\infty (1-\eta\beta_2(2,\eta)y)^{1/\eta} dy\\
&=\frac{4t\beta_1(2,\eta)}{\beta_2(2,\eta)}\frac{1}{\eta+1}\left(1-\eta\beta_2(2,\eta)\frac{(K-1)\sqrt{r^2+t^2}}{\|a'\|}\right)^{\frac{\eta+1}{\eta}}.
\end{split}
\end{equation}
Note that $\|a'\|\le r$. Finally, we have
\begin{equation*}
\begin{split}
\Pr_{X\sim \cD_{u,t}}(|a\cdot x|> K\sqrt{r^2+t^2})&\le \frac{4\beta_1(2,\eta)}{f_2(s,n)\beta_2(2,\eta)}\frac{1}{\eta+1}\left(1-\eta\beta_2(2,\eta)\frac{(K-1)\sqrt{r^2+t^2}}{r}\right)^{\frac{\eta+1}{\eta}}\\
&\le \frac{4\beta_1(2,\eta)}{f_2(s,n)\beta_2(2,\eta)}\frac{1}{\eta+1}\left(1-c\eta\beta_2(2,\eta)\frac{K\sqrt{r^2+t^2}}{r}\right)^{\frac{\eta+1}{\eta}},
\end{split}
\end{equation*}
for an absolute constant $c$.
\end{proof}

\noindent{\textbf{Theorem \ref{theorem: 1-D variance}} (restated)}
\emph{Assume that $\cD$ is isotropic $s$-concave. For $d$ given by Theorem \ref{theorem: upper and lower bounds} (a), there is an absolute $C_0$ such that for all $0<t\le d$ and for all $a$ such that $\|u-a\|\le r$ and $\|a\|\le 1$,
$
\E_{X\sim \cD_{u,t}}[(a\cdot x)^2]\le f_5(s,n)(r^2+t^2),
$
where $f_5(s,n)=16+C_0\frac{8\beta_1(2,\eta)B(-1/\eta-3,2)}{f_2(s,n)\beta_2(2,\eta)^3(\eta+1)\eta^2}$, $(\beta_1(2,\eta), \beta_2(2,\eta))$ and $f_2(s,n)$ are given by Lemma \ref{lemma: more refined upper bound} and Theorem \ref{theorem: probability within margin}, respectively, and $\eta=\frac{s}{1+(n-2)s}$.}

\medskip
\begin{proof}
Denote by $z=\sqrt{r^2+t^2}$. Then we have
\begin{equation}
\begin{split}
\E_{x\sim \cD_{u,t}}[(a\cdot x)^2]&=\int_0^\infty \Pr_{x\sim \cD_{u,t}}[(a\cdot x)^2\ge z]dz\\
&\le 16z^2+\int_{16z^2}^\infty \Pr_{x\sim \cD_{u,t}}[(a\cdot x)^2\ge z]dz\\
&\le 16z^2+\frac{4\beta_1(2,\eta)}{f_2(s,n)\beta_2(2,\eta)}\frac{1}{\eta+1}\int_0^\infty\left(1-\eta\beta_2(2,\eta)\frac{c\sqrt{z}}{r}\right)^{\frac{\eta+1}{\eta}}dz\\
&=16z^2+\frac{8\beta_1(2,\eta)}{f_2(s,n)\beta_2(2,\eta)}\frac{1}{\eta+1}\int_0^\infty y\left(1-\eta\beta_2(2,\eta)\frac{cy}{r}\right)^{\frac{\eta+1}{\eta}}dy\\
&=16z^2+\frac{8\beta_1(2,\eta)}{f_2(s,n)\beta_2(2,\eta)}\frac{1}{\eta+1}C_0B(-1/\eta-3,2)\frac{r^2}{\eta^2\beta_2(2,\eta)^2}\\
&=\left(16+C_0\frac{8\beta_1(2,\eta)B(-1/\eta-3,2)}{f_2(s,n)\beta_2(2,\eta)^3(\eta+1)\eta^2}\right)r^2+16t^2\\
&\le \left(16+C_0\frac{8\beta_1(2,\eta)B(-1/\eta-3,2)}{f_2(s,n)\beta_2(2,\eta)^3(\eta+1)\eta^2}\right)(r^2+t^2),
\end{split}
\end{equation}
where $c$, $C_0$ are absolute constants.
\end{proof}

\section{Proof of Theorem \ref{theorem: margin-based active learning realizable case}}
\label{section: proof of margin-based active learning realizable case}

\noindent{\textbf{Theorem \ref{theorem: margin-based active learning realizable case}} (restated)}
\emph{In the realizable case, let $\cD$ be an isotropic $s$-concave distribution in $\R^n$. There exist constants $C$ and $c$ such that for any $0<\epsilon<1/4$ and $\delta>0$, Algorithm \ref{algorithm: margin based active learning realizable case} with $b_k=\min\{\Theta(2^{-k}f_4f_1^{-1}),d\}$ and $m_k=C\left(\frac{f_3b_{k-1}}{2^{-k}}\left(n\log \frac{f_3b_{k-1}}{2^{-k}}+\log \frac{1+s-k}{\delta}\right)\right)$, after $T=\lceil\log \frac{1}{c\epsilon}\rceil$ iterations, outputs a linear separator of error at most $\epsilon$ with probability at least $1-\delta$.}

\medskip
\begin{proof}
We will show by induction that for all $k\le s$, with probability at least $1-\frac{\delta}{2}\sum_{i<k}\frac{1}{(1+s-i)^2}$, any $w$ that is consistent with the examples in $W(k)$, e.g. $w_k$, has error at most $c2^{-k}$.

The case of $k=1$ follows from the VC theory (Theorem \ref{theorem: vc theory}). Assume now that the claim is true for $k-1$. We now consider the $k$th iteration. Denote by $S_{k-1}=\{x: |w_{k-1}\cdot x|\le b_{k-1}\}$ and $\bar S_{k-1}=\{x: |w_{k-1}\cdot x|> b_{k-1}\}$. By the induction hypothesis, with probability at least $1-\frac{\delta}{2}\sum_{i<k-1}\frac{1}{(1+s-i)^2}$, any $w$ that is consistent with $W(k-1)$, including $w_{k-1}$, has error at most $c2^{-(k-1)}$. For such a hypothesis $w$ and $w_{k-1}$, by Theorem \ref{theorem: disagreement and angle}, we have $\theta(w,w^*)\le cf_1^{-1}2^{-(k-1)}$ and $\theta(w_{k-1},w^*)\le cf_1^{-1}2^{-(k-1)}$. Thus $\theta(w_{k-1},w)\le \theta(w_{k-1},w^*)+\theta(w^*,w)\le 4cf_1^{-1} 2^{-k}$. So by Theorem \ref{theorem: disagreement outside band}, there is a choice of band width $b_{k-1}=\min\{\Theta(f_4f_1^{-1}2^{-k}),d\}$ such that $\Pr(\sign(w\cdot x)\not=\sign(w_{k-1}\cdot x),x\in \bar S_{k-1})\le \frac{c2^{-k}}{4}$ and $\Pr[\sign(w_{k-1}\cdot x)\not=\sign(w^*\cdot x),x\in \bar S_{k-1}]\le \frac{c2^{-k}}{4}$. Therefore,
$\Pr[\sign(w\cdot x)\not=\sign(w^*\cdot x), x\in\bar S_{k-1}]\le \frac{c2^{-k}}{2}$.

We now consider the case when $x\in S_{k-1}$. By Algorithm \ref{algorithm: margin based active learning realizable case}, we label $m_k$ data points in $S_{k-1}$ at the $(k-1)$th iteration. So according to the VC theory (Theorem \ref{theorem: vc theory}), with probability at least $1-\delta/(4(1+s-k)^2)$, for all $w$ that is consistent with the examples in $W(k)$,
$
\err(w|S_{k-1})=\Pr[\sign(w\cdot x)\not=\sign(w^*\cdot x)|x\in S_{k-1}]\le \frac{c2^{-k}}{2b_{k-1}f_3}
$.
Finally, note that Theorem \ref{theorem: probability within margin} implies that $\Pr(S_{k-1})\le f_3b_{k-1}$. So we have
$\err(w)=\Pr[\sign(w\cdot x)\not=\sign(w^*\cdot x), x\in\bar S_{k-1}]+\Pr[\sign(w\cdot x)\not=\sign(w^*\cdot x), x\in S_{k-1}]\le \frac{c2^{-k}}{2}+\frac{c2^{-k}}{2b_{k-1}f_3}\times f_3b_{k-1}=c2^{-k}$.
The proof is completed.
\end{proof}

\section{Proof of Theorem \ref{theorem: margin-based active learning adversarial noise}}
\label{appendix: adversarial noise induction}

Before proceeding, let $\ell_\tau(w,x,y)=\max\{0,1-y(w\cdot x)/\tau\}$, $\ell_\tau(w,T)\hspace{-0.1cm}=\hspace{-0.1cm}\frac{1}{|T|}\sum_{(x,y)\in T}\hspace{-0.1cm}\ell_\tau(w,x,y)$, and $L_\tau(w,\cD)=\E_{x\sim\cD}\ell_\tau(w,x,\sign(w^*\cdot x))$. Our analysis will involve the distribution $\cD_{w,t}$ obtained by conditioning $\cD$ on membership in the band, namely, the set $\{x:|w\cdot x|\le t\}$.

\begin{lemma}
\label{lemma: upper bound of expected loss}
$L_{\tau_k}(w^*,\cD_{w_{k-1},b_{k-1}})\le \kappa/6$, if $\kappa\ge \frac{6f_3\tau_k}{f_2b_{k-1}}$ and $b_{k-1}\le d$.
\end{lemma}

\begin{proof}
Note that $y(w^*\cdot x)$ cannot be negative on any clean example $(x,y)$. So we have
$\ell(w^*,x,y)=\max\{0,1-y(w^*\cdot x)/\tau_k\}\le 1$ and $w^*$ pays a non-zero hinge loss only inside the margin $\{x: |w^*\cdot x|\le \tau_k\}$. Thus $L_{\tau_k}(w^*,\cD_{w_{k-1},b_{k-1}})\le \Pr_{\cD_{w_{k-1},b_{k-1}}}(|w^*\cdot x|\le \tau_k)=\Pr_\cD(|w^*\cdot x|\le \tau_k,|w_{k-1}\cdot x|\le b_{k-1})/\Pr_\cD(|w_{k-1}\cdot x|\le b_{k-1})$. Notice that the numerator can be bounded by $\Pr_\cD(|w^*\cdot x|\le \tau_k)\le f_3\tau_k$ according to Theorem \ref{theorem: probability within margin}. As for the denominator, by Theorem \ref{theorem: probability within margin} we have $\Pr_\cD(|w_{k-1}\cdot x|\le b_{k-1})\ge f_2\min\{b_{k-1},d\}$. So we have $L_{\tau_k}(w^*,\cD_{w_{k-1},b_{k-1}})\le f_3\tau_k/(f_2\min\{b_{k-1},d\})\le \kappa/6$.
\end{proof}

Let $\widetilde \cP_k$ be the noisy distribution of $(x,y)$ where $x\sim \cD_{w_{k-1},b_{k-1}}$ and $y$ obeys the adversarial noise model, and denote by $\cP_k$ the clean distribution of $(x,y)$ where $x\sim \cD_{w_{k-1},b_{k-1}}$ and $y=\sign(w^*\cdot x)$. The following key lemma bounds the distance of expected loss w.r.t. the distributions $\widetilde \cP_k$ and $\cP_k$.
\begin{lemma}
\label{lemma: distance between clean and noisy distribution}
There exists an absolute constant $c$ such that for any $w\in \ball(w_{k-1},r_k)$, we have that
$
\left|\E_{(x,y)\sim \cP_k}\ell(w,x,y)-\E_{(x,y)\sim \widetilde{\cP}_k}\ell(w,x,y)\right|\le \frac{2}{\tau_k}\sqrt{\frac{\eta f_5(r_k^2+b_{k-1}^2)}{f_2 b_{k-1}}}.
$
\end{lemma}

\begin{proof}
Denote by $N$ the set of noisy examples. Let $\widetilde \cP$ be the noisy distribution of $(x,y)$ where $x\sim \cD$ and $y$ obeys the adversarial noise model. We have
\begin{equation*}
\begin{split}
&  \left| \E_{(x,y)\sim \widetilde \cP_k} [\ell_{\tau_k}(w^*, x, y)] - \E_{(x,y)\sim  \cP_k} [\ell_{\tau_k}(w^*, x, y)] \right| \\
& \qquad \leq
\left| \E_{(x,y)\sim  \widetilde \cP_k} \left[ \mathbf{1}_{x\in N} \left( \ell_{\tau_k}(w^*, x, y) - \ell_{\tau_k}(w^*,  x, \sign(w^*\cdot x) )\right) \right] \right| \\
& \qquad \leq 2 ~ \E_{(x,y)\sim  \widetilde \cP_k} \left[ \mathbf{1}_{x\in N} \left(\frac{|w^*\cdot x|}{\tau_k} \right) \right] \\
&\qquad \leq
\frac{2}{\tau_k} \sqrt{\Pr_{(x,y)\sim \widetilde \cP_k}[x\in N]}  \times  \sqrt{\E_{(x,y)\sim \widetilde \cP_k}[(w^*\cdot x)^2]} \qquad \text{(By Cauchy Schwarz)}\\
&\qquad \le \frac{2}{\tau_k}\sqrt{\frac{\eta}{\Pr_{\widetilde \cP}(|w_{k-1}\cdot x|\le b_{k-1})}}\times \sqrt{f_5(r_k^2+b_{k-1}^2)} \qquad \text{(By Theorem \ref{theorem: 1-D variance})}\\
&\qquad \le \frac{2}{\tau_k}\sqrt{\frac{\eta f_5(r_k^2+b_{k-1}^2)}{f_2 b_{k-1}}}. \qquad \text{(By Theorem \ref{theorem: probability within margin})}\\
\end{split}
\end{equation*}
\end{proof}

\begin{lemma}
\label{lemma: concentration}
Denote by $W$ the samples drawn from the noisy distribution $\widetilde \cP_k$ and suppose that
$
|W|=O\left( \frac{[b_{k-1}s+\tau_k(1+ns)\sqrt{n}[1-(\delta/(\sqrt{n}(k+k^2)))^{s/(1+ns)}]+\tau_ks]^2}{\kappa^2\tau_k^2s^2}\left(n+\log \frac{k+k^2}{\delta}\right)\right).
$
Then with probability at least $1-\frac{\delta}{k+k^2}$, for all $w\in \ball(w_{k-1},r_k)$, we have
\begin{equation*}
\left|\E_{(x,y)\sim \widetilde\cP_k}\ell(w,x,y)-\ell(w,W)\right|\le\kappa/16.
\end{equation*}
\end{lemma}

\begin{proof}
To establish the lemma, we apply some standard VC tools (Theorem \ref{theorem: pseudo-dimension argument}). Note that the pseudo-dimension of $\{\ell(w,\cdot):w\in\R^n\}$ is at most $n$~\cite{awasthi2014power}. To use Theorem \ref{theorem: pseudo-dimension argument}, we first provide the upper bound on the loss. On one hand, note that
\begin{equation*}
\begin{split}
\ell(w,x,y)&\le 1+\frac{|w\cdot x|}{\tau_k}\le 1+\frac{|w_{k-1}\cdot x|+\|w-w_{k-1}\|\|x\|}{\tau_k}\\
&\le 1+\frac{b_{k-1}+\tau_k\|x\|}{\tau_k}.
\end{split}
\end{equation*}
On the other hand, by Theorem \ref{theorem: tail probability} and the union bound, with probability at least $1-\frac{\delta}{k+k^2}$, we have that
$
\max_{x\in W} \|x\|\le C\frac{(1+ns)\sqrt{n}}{s}\left\{1-\left[\frac{\delta}{6(k+k^2)|W|}\right]^{s/(1+ns)}\right\},
$
for an absolute constant $C$. The conclusion then follows from Theorem \ref{theorem: pseudo-dimension argument}.
\end{proof}

\begin{lemma}
\label{lemma: adversarial noise induction}
Let $k\le\lceil\log (1/(c\epsilon))\rceil$ where $c$ is an absolute constant. If $\kappa=\max\left\{\frac{f_3\tau_k}{f_2\min\{b_{k-1},d\}},\frac{b_{k-1}\sqrt{f_5}}{\tau_k\sqrt{f_2}}\right\}$, $r_k\le O(b_{k-1})$, $\eta\le O(b_{k-1})$, $m_k=O\left( \frac{[b_{k-1}s+\tau_k(1+ns)\sqrt{n}[1-(\delta/(k+k^2))^{s/(1+ns)}]+\tau_ks]^2}{\kappa^2\tau_k^2s^2}\left(n+\log \frac{k+k^2}{\delta}\right)\right)$, and $b_{k-1}\le d$, then with probability at least $1-\frac{\delta}{k+k^2}$, we have $\err_{\cD_{w_{k-1},b_{k-1}}}(w_k)\le \kappa$.
\end{lemma}

\medskip
\begin{proof}
With probability at least $1-\frac{\delta}{k+k^2}$, we have
\begin{equation*}
\begin{split}
\err_{\cD_{w_{k-1},b_{k-1}}}(w_k)&=\err_{\cD_{w_{k-1},b_{k-1}}}(v_k)\\
&\le \E_{(x,y)\sim \cP_k}\ell(v_k,x,y)\\
&\le \E_{(x,y)\sim \widetilde \cP_k}\ell(v_k,x,y)+\frac{2}{\tau_k}\sqrt{\frac{\eta f_5(r_k^2+b_{k-1}^2)}{f_2 b_{k-1}}} \qquad\text{(By Lemma \ref{lemma: distance between clean and noisy distribution})}\\
&\le \ell(v_k,W)+\frac{2}{\tau_k}\sqrt{\frac{\eta f_5(r_k^2+b_{k-1}^2)}{f_2 b_{k-1}}}+\frac{\kappa}{16} \qquad\text{(By Lemma \ref{lemma: concentration})}\\
&\le \ell(w^*,W)+\frac{4}{\tau_k}\sqrt{\frac{\eta f_5(r_k^2+b_{k-1}^2)}{f_2 b_{k-1}}}+\frac{\kappa}{8}\qquad \text{(Since $\|v_k\|\ge 1/2$)}\\
&\le \E_{(x,y)\sim \widetilde \cP_k}\ell(w^*,x,y)+\frac{4}{\tau_k}\sqrt{\frac{\eta f_5(r_k^2+b_{k-1}^2)}{f_2 b_{k-1}}}+\frac{\kappa}{4} \qquad\text{(By Lemma \ref{lemma: concentration})}\\
&\le \E_{(x,y)\sim \cP_k}\ell(w^*,x,y)+\frac{6}{\tau_k}\sqrt{\frac{\eta f_5(r_k^2+b_{k-1}^2)}{f_2 b_{k-1}}}+\frac{\kappa}{4} \qquad\text{(By Lemma \ref{lemma: distance between clean and noisy distribution})}\\
&\le \frac{6}{\tau_k}\sqrt{\frac{\eta f_5(r_k^2+b_{k-1}^2)}{f_2 b_{k-1}}}+\frac{\kappa}{2} \qquad\text{(By Lemma \ref{lemma: upper bound of expected loss})}\\
&\le \kappa,
\end{split}
\end{equation*}
where the last inequality holds because $\kappa\tau_k\sqrt{\frac{f_2}{f_5}}\ge \Theta(b_{k-1})$, $r_k\le O(b_{k-1})$, and $\eta\le O(b_{k-1})$.
\end{proof}

Now we are ready to prove Theorem \ref{theorem: margin-based active learning adversarial noise}.

\medskip
\noindent{\textbf{Theorem \ref{theorem: margin-based active learning adversarial noise}} (restated)}
\emph{Let $\cD$ be an isotropic $s$-concave distribution in $\R^n$ and the label $y$ obeys the adversarial noise model. If the rate $\eta$ of adversarial noise satisfies $\eta<c_0\epsilon$ for some absolute constant $c_0$, then there exists an absolute constant $c$ such that for any $0<\epsilon<1/4$ and $\delta>0$, Algorithm \ref{algorithm: margin based active learning adversarial noise} with $b_k=\min\{\Theta(2^{-k}f_4f_1^{-1}),d\}$, $\tau_k=\Theta\left(f_1^{-2}f_2^{-1/2}f_3f_4^2f_5^{1/2}2^{-(k-1)}\right)$, $r_k=\Theta(2^{-k}f_1^{-1})$, $m_k=O\left( \frac{[b_{k-1}s+\tau_k(1+ns)\sqrt{n}[1-(\delta/(k+k^2))^{s/(1+ns)}]+\tau_ks]^2}{\kappa^2\tau_k^2s^2}\left(n+\log \frac{k+k^2}{\delta}\right)\right)$, and $\kappa=\max\left\{\frac{f_3\tau_k}{f_2\min\{b_{k-1},d\}},\frac{b_{k-1}\sqrt{f_5}}{\tau_k\sqrt{f_2}}\right\}$, after $T=\lceil\log \frac{1}{c\epsilon}\rceil$ iterations, outputs a linear separator $w_T$ such that $\Pr_{x\sim \cD}[\sign(w_T\cdot x)\not=\sign(w^*\cdot x)]\le\epsilon$ with probability at least $1-\delta$.}

\medskip
\begin{proof}
The case of $k=1$ is obvious. Assume now that the claim is true for $k-1$. We now consider the $k$th iteration. Denote by $S_{k-1}=\{x: |w_{k-1}\cdot x|\le b_{k-1}\}$ and $\bar S_{k-1}=\{x: |w_{k-1}\cdot x|> b_{k-1}\}$. By the induction hypothesis, with probability at least $1-\frac{\delta}{2}\sum_{i<k-1}\frac{1}{(1+s-i)^2}$, $w_{k-1}$ has error at most $c2^{-(k-1)}$. Then by Theorem \ref{theorem: disagreement and angle}, we have $\theta(w_{k-1},w^*)\le cf_1^{-1}2^{-(k-1)}$. On the other hand, since $\|w_{k-1}\|=1$ and $v_k\in B(w_{k-1},r_k)$, we have $\theta(w_{k-1},v_k)\le r_k$. This in turn implies $\theta(w_{k-1},w_k)\le 2^{-k}f_1^{-1}$. So by Theorem \ref{theorem: disagreement outside band}, there is a choice of band width $2b_{k-1}=O(f_4f_1^{-1}2^{-k})$ such that $\Pr(\sign(w_k\cdot x)\not=\sign(w_{k-1}\cdot x),x\in \bar S_{k-1})\le \frac{c2^{-k}}{4}$ and $\Pr[\sign(w_{k-1}\cdot x)\not=\sign(w^*\cdot x),x\in \bar S_{k-1}]\le \frac{c2^{-k}}{4}$. Therefore,
$\Pr[\sign(w_k\cdot x)\not=\sign(w^*\cdot x), x\in\bar S_{k-1}]\le \frac{c2^{-k}}{2}$. Finally, note that Theorem \ref{theorem: probability within margin} implies that $\Pr(S_{k-1})\le f_3b_{k-1}$. So we have
$
\err_\cD(w_k)=\Pr[\sign(w_k\cdot x)\not=\sign(w^*\cdot x), x\in\bar S_{k-1}]+\Pr[\sign(w_k\cdot x)\not=\sign(w^*\cdot x), x\in S_{k-1}]
=\Pr[\sign(w_k\cdot x)\not=\sign(w^*\cdot x), x\in\bar S_{k-1}]+\err_{\cD_{w_{k-1},b_{k-1}}}(w_k)\Pr(x\in S_{k-1})
\le \frac{c2^{-k}}{2}+\kappa\times f_3b_{k-1}\le c2^{-k}=\epsilon,
$
where the penultimate inequality follows from Lemma \ref{lemma: adversarial noise induction}. The proof is completed.
\end{proof}

\section{Proof of Theorem \ref{theorem: disagreement coefficient}}

\noindent{\textbf{Theorem \ref{theorem: disagreement coefficient}} (restated)}
\emph{Let $\mathcal{D}$ be an isotropic $s$-concave distribution over $\R^n$. Then for any $w^*\in\R^n$ and $r>0$, the disagreement coefficient is $\Theta_{w^*,\mathcal{D}}(\epsilon)=O\left(\sqrt{n}\frac{(1+ns)^2}{s(1+(n+2)s)f_1(s,n)}(1-\epsilon^{s/(1+ns)})\right)$, where $f_1(s,n)$ is given by Theorem \ref{theorem: disagreement and angle}. In particular, when $s\rightarrow0$ (a.k.a. log-concave), $\Theta_{w^*,\mathcal{D}}(\epsilon)=O(\sqrt{n}\log(1/\epsilon))$.}
\medskip
\begin{proof}
Consider any unit $w$ such that $d_{\mathcal{D}}(w,w^*)\le r$. According to Theorem \ref{theorem: disagreement and angle}, we have $\|w-w^*\|<\theta(w,w^*)\le d_{\mathcal{D}}(w,w^*)/f(s)\le r/f(s)$. Thus for any $x$ such that $\|x\|\le O(\sqrt{n}\frac{1+ns}{s}(1-r^{s/(1+ns)}))$, we have
$
w\cdot x-w^*\cdot x\le \|w-w^*\|\ \|x\|<r\sqrt{n}\frac{1+ns}{sf(s)}(1-r^{s/(1+ns)})$.
So as soon as $|w^*\cdot x|\ge r\sqrt{n}\frac{1+ns}{sf(s)}(1-r^{s/(1+ns)})$, we will have $\text{sign}(w\cdot x)=\text{sign}(w^*\cdot x)$, i.e., $w$ and $w^*$ agree with each other. We now evaluate the probability. By Theorem \ref{theorem: probability within margin},
$\Pr_{x\sim\mathcal{D}}\left[|w^*\cdot x|\le r\sqrt{n}\frac{1+ns}{sf(s)}(1-r^{s/(1+ns)})\right]\le 2\frac{1+ns}{1+(n+2)s}r\sqrt{n}\frac{1+ns}{sf(s)}(1-r^{s/(1+ns)})$.
Moreover, $\Pr_{x\sim\mathcal{D}}\left[\|x\|\ge c\sqrt{n}\frac{1+ns}{s}(1-r^{s/(1+ns)})\right]\le Cr$ by Theorem \ref{theorem: tail probability}.
Thus
\begin{equation*}
\begin{split}
\text{cap}_{w^*,\mathcal{D}}(r)&\le \frac{\Pr_{x\sim\mathcal{D}}\left[|w^*\cdot x|\le r\sqrt{n}\frac{1+ns}{sf(s)}(1-r^{s/(1+ns)})\right]}{r}+\frac{\Pr_{x\sim\mathcal{D}}\left[\|x\|\ge c\sqrt{n}\frac{1+ns}{s}(1-r^{s/(1+ns)})\right]}{r}\\
&=O\left(\sqrt{n}\frac{(1+ns)^2}{s(1+(n+2)s)f(s)}(1-r^{s/(1+ns)})\right).
\end{split}
\end{equation*}
Therefore,
$\Theta_{w^*,\mathcal{D}}(\epsilon)=\sup_{r\ge\epsilon}[\text{cap}_{w^*,\mathcal{D}}(r)]=O\left(\sqrt{n}\frac{(1+ns)^2}{s(1+(n+2)s)f(s)}(1-\epsilon^{s/(1+ns)})\right)$.
\end{proof}

\section{Proof of Theorem \ref{theorem: Baum's algorithm}}
\begin{lemma}
\label{lemma: isotropic reflection}
Denote by $R$ the intersections of three origin-centered halfspaces in $\R^n$. Suppose that the instance $x$ in $\R^n$ is drawn from an isotropic $s$-concave distribution. Then $\Pr[x\in-R]\le K\Pr[x\in R]$, where $K=\beta_1(3,\kappa)\frac{B(-1/\kappa-3,3)}{(-\kappa\beta_2(3,\kappa))^3}\frac{3+1/\kappa}{h(\kappa)d^{3+1/\kappa}}$,
$\beta_1(3,\kappa)$, $\beta_2(3,\kappa)$, and $a(3,\kappa)$ are as in Lemma \ref{lemma: more refined upper bound},
$h(\kappa)=\left(\frac{1}{d}((2-2^{-4\kappa})^{-1}-1)+1\right)^{1/\kappa}(4e\pi)^{-3/2}\left[\left(\frac{1+\beta}{1+3\beta}\sqrt{3(1+\gamma)^{3/\gamma}}2^{2+1/\kappa}\right)^{\kappa}-1\right]^{-1/\kappa}$,
$d=(1+\gamma)^{-1/\gamma}\frac{1+3\beta}{3+3\beta}$, $\beta=\frac{\kappa}{1+2\kappa}$, $\gamma=\frac{\kappa}{1+\kappa}$, and $\kappa=s/(1+(n-3)s)$.
\end{lemma}

\medskip
\begin{proof}
Let $u_1$, $u_2$, and $u_3$ be normals to the hyperplanes bounding the region $R$, namely $R=\{x\in\R^n: u_1\cdot x\ge0\ \mbox{and}\ u_2\cdot x\ge 0\ \mbox{and}\ u_3\cdot x\ge 0\}$. Denote by $U$ the linear span of vectors $u_1$, $u_2$, and $u_3$, and let $(e_1,e_2,e_3)$ be an orthogonal basis of $U$ and $(e_1,e_2,e_3,...,e_n)$ be an extension of basis $(e_1,e_2,e_3)$ to $\R^n$. Represent the components of $x$, $u_1$, $u_2$, and $u_3$ in term of basis $(e_1,e_2,e_3,...,e_n)$ as
\begin{equation*}
x=(x_1,x_2,x_3,x_4,...,x_n),
\end{equation*}
\begin{equation*}
u_1=(u_{1,1},u_{1,2},u_{1,3},0,...,0),
\end{equation*}
\begin{equation*}
u_2=(u_{2,1},u_{2,2},u_{2,3},0,...,0),
\end{equation*}
\begin{equation*}
u_3=(u_{3,1},u_{3,2},u_{3,3},0,...,0).
\end{equation*}
Denote by $\textsf{proj}_U(x)\triangleq(x_1,x_2,x_3)$ the projection of $x$ onto subspace $U$, and let $\textsf{proj}_U(R)$ be the projection of $R$ onto $U$. Because the dot products of a point with normal vectors of $R$ are all that is needed to determine the membership in $R$, we have
\begin{equation}
\label{equ: membership in original and projected space}
\begin{split}
x\in R&\Leftrightarrow u_{j,1}x_1+u_{j,2}x_2+u_{j,3}x_3\ge 0\mbox{ for all }j\in\{1,2,3\}\\
&\Leftrightarrow \textsf{proj}_U(x)\in\textsf{proj}_U(R).
\end{split}
\end{equation}
Let $f$ be the density of the isotropic $s$-concave distribution and $g$ be the marginal density of $f$ w.r.t. $(x_1,x_2,x_3)$. Thus by \eqref{equ: membership in original and projected space},
\begin{equation*}
\begin{split}
\Pr[x\in R]&={\int\cdots\int}_R f(x_1,x_2,x_3,x_4,...,x_n)dx_1...dx_n\\
&={\int\int\int}_{\textsf{proj}_U(R)} g(x_1,x_2,x_3)dx_1dx_2dx_3.
\end{split}
\end{equation*}
Note that $f$ is isotropic and $s$-concave. So according to Theorem \ref{theorem: marginal}, $g$ is isotropic and $\kappa$-concave with $\kappa=s/(1+(n-3)s)$. We now use Theorem \ref{theorem: upper and lower bounds} and Lemma \ref{lemma: more refined upper bound} to bound $g$. Specifically, let $u\triangleq(x_1,x_2,x_3)$. On one hand, according to Theorem \ref{theorem: upper and lower bounds} (a) and (d), for any $u\in\R^3$ such that $\|u\|\le d$,
\begin{equation*}
\begin{split}
&g(u)\ge \left(\frac{\|u\|}{d}((2-2^{-4\kappa})^{-1}-1)+1\right)^{1/\kappa}f(0)\\
&>\left(\frac{\|u\|}{d}((2-2^{-4\kappa})^{-1}-1)+1\right)^{1/\kappa}(4e\pi)^{-3/2}\left[\left(\frac{1+\beta}{1+3\beta}\sqrt{3(1+\gamma)^{3/\gamma}}2^{2+1/\kappa}\right)^{\kappa}-1\right]^{-1/\kappa}\\
&\triangleq \|u\|^{1/\kappa}h(\kappa),
\end{split}
\end{equation*}
where $d=(1+\gamma)^{-1/\gamma}\frac{1+3\beta}{3+3\beta}$, $\beta=\frac{\kappa}{1+2\kappa}$, $\gamma=\frac{\kappa}{1+\kappa}$, and
\begin{equation*}
h(\kappa)=\left(\frac{1}{d}((2-2^{-4\kappa})^{-1}-1)+1\right)^{1/\kappa}(4e\pi)^{-3/2}\left[\left(\frac{1+\beta}{1+3\beta}\sqrt{3(1+\gamma)^{3/\gamma}}2^{2+1/\kappa}\right)^{\kappa}-1\right]^{-1/\kappa}.
\end{equation*}
On the other hand, by Lemma \ref{lemma: more refined upper bound}, for every $u\in\R^3$,
\begin{equation*}
g(u)\le\beta_1(3,\kappa)(1-\kappa\beta_2(3,\kappa)\|u\|)^{1/\kappa},
\end{equation*}
where
\begin{equation*}
\red{\beta_1(3,\kappa)=(2-2^{-4\kappa})^{1/\kappa}\frac{1}{2\pi^{3/2}d^3}(1-\kappa)^{-1/\kappa}3\Gamma(3/2)\left[\left(\frac{1+\beta}{1+3\beta}\sqrt{3(1+\gamma)^{3/\gamma}}2^{2+1/\kappa}\right)^{\kappa}-1\right]^{1/\kappa},}
\end{equation*}
\begin{equation*}
\red{\beta_2(3,\kappa)=\frac{2\pi d^2}{2}(2-2^{-3s})^{-1/s}\frac{[(a+(1-s)\beta_1(3,\kappa)^\kappa)^{1+1/\kappa}-a^{1+1/\kappa}]\kappa}{\beta_1(3,\kappa)^s(1+\kappa)(1-\kappa)},}
\end{equation*}
and
\begin{equation*}
a=(4e\pi)^{-3\kappa/2}\left[\left(\frac{1+\beta}{1+3\beta}\sqrt{3(1+\gamma)^{3/\gamma}}2^{2+1/\kappa}\right)^{\kappa}-1\right]^{-1}.
\end{equation*}
Denote by $R'=\textsf{proj}_U(R)\cap \ball(0,d)$, and $\ball(0,d)$ is the origin-centered ball of radius $d$ in $\R^3$. Thus we have
\begin{equation*}
\begin{split}
\int\int\int_{R'}\|u\|^{1/\kappa}h(\kappa)&du_1du_2du_3\le \Pr[x\in R]\\
&\le \int\int\int_{\textsf{proj}_U(R)}\beta_1(3,\kappa)(1-\kappa\beta_2(3,\kappa)\|u\|)^{1/\kappa}du_1du_2du_3.
\end{split}
\end{equation*}
Let $A\triangleq\int\int_{\textsf{proj}_U(R)\cap\mathbb{S}^2} sin\theta d\varphi d\theta$. Note that
\begin{equation*}
\int\int\int_{R'}\|u\|^{1/\kappa}h(\kappa)du_1du_2du_3=A\int_0^d r^2r^{1/\kappa}h(\kappa)dr=Ah(\kappa)\frac{1}{3+1/\kappa}d^{3+1/\kappa},
\end{equation*}
and
\begin{equation*}
\begin{split}
&\int\int\int_{\textsf{proj}_U(R)}\beta_1(3,\kappa)(1-\kappa\beta_2(3,\kappa)\|u\|)^{1/\kappa}du_1du_2du_3\\
&=A\beta_1(3,\kappa)\int_0^\infty r^2(1-\kappa\beta_2(3,\kappa) r)^{1/\kappa}dr\\
&=A\beta_1(3,\kappa)\frac{B(-1/\kappa-3,3)}{(-\kappa\beta_2(3,\kappa))^3}.
\end{split}
\end{equation*}
So we have
\begin{equation*}
Ah(\kappa)\frac{1}{3+1/\kappa}d^{3+1/\kappa}\le\Pr[x\in R]\le A\beta_1(3,\kappa)\frac{B(-1/\kappa-3,3)}{(-\kappa\beta_2(3,\kappa))^3},
\end{equation*}
and by symmetry,
\begin{equation*}
Ah(\kappa)\frac{1}{3+1/\kappa}d^{3+1/\kappa}\le\Pr[x\in -R]\le A\beta_1(3,\kappa)\frac{B(-1/\kappa-3,3)}{(-\kappa\beta_2(3,\kappa))^3}.
\end{equation*}
Therefore,
\begin{equation*}
\Pr[x\in -R]\le \Pr[x\in R]\beta_1(3,\kappa)\frac{B(-1/\kappa-3,3)}{(-\kappa\beta_2(3,\kappa))^3}\frac{3+1/\kappa}{h(\kappa)d^{3+1/\kappa}}.
\end{equation*}
\end{proof}

\noindent{\textbf{Theorem \ref{theorem: Baum's algorithm}} (restated)}
\emph{In the PAC realizable case, Algorithm \ref{algorithm: intersections of halfspaces} outputs a hypothesis $h$ of error at most $\epsilon$ with probability at least $1-\delta$ under isotropic $s$-concave distribution. The label complexity is $M(\epsilon/2,\delta/4,n^2)+\max\{2m_2/\epsilon,(2/\epsilon^2)\log(4/\delta)\}$, where $M(\epsilon,\delta,m)$ is defined by $M(\epsilon,\delta,n)=O\left(\frac{n}{\epsilon}\log \frac{1}{\epsilon}+\frac{1}{\epsilon}\log\frac{1}{\delta}\right)$, $m_2=M(\max\{\delta/(4eKm_1),\epsilon/2\},\delta/4,n)$, $K=\beta_1(3,\kappa)\frac{B(-1/\kappa-3,3)}{(-\kappa\beta_2(3,\kappa))^3}\frac{3+1/\kappa}{h(\kappa)d^{3+1/\kappa}}$, $d=(1+\gamma)^{-1/\gamma}\frac{1+3\beta}{3+3\beta}$, $h(\kappa)=\left(\frac{1}{d}((2-2^{-4\kappa})^{-1}-1)+1\right)^{\frac{1}{\kappa}}(4e\pi)^{-\frac{3}{2}}\left[\left(\frac{1+\beta}{1+3\beta}\sqrt{3(1\hspace{-0.1cm}+\hspace{-0.1cm}\gamma)^{3/\gamma}}2^{2+\frac{1}{\kappa}}\right)^{\kappa}\hspace{-0.2cm}-\hspace{-0.1cm}1\right]^{-1/\kappa}$\hspace{-0.2cm}, $\beta=\frac{\kappa}{1+2\kappa}$, $\gamma=\frac{\kappa}{1+\kappa}$, and $\kappa=\frac{s}{1+(n-3)s}$. In particular, when $s\rightarrow 0$ (a.k.a. log-concave), $K$ is an absolute constant.}

\medskip
\begin{proof}
Denote by $p$ the probability of observing a positive example. We discuss the following three cases.

\noindent{\textbf{1.} $r<m_2$ and $p<\epsilon$\textbf{.}}

In this case, the hypothesis that labels every examples as negative has error less than $\epsilon$. Therefore, the algorithm behaves with error at most $\epsilon$ when $r<m_2$.

\noindent{\textbf{2.} $r<m_2$ and $p\ge\epsilon$\textbf{.}}

By the Hoeffding inequality,
\begin{equation*}
\Pr(r<m_2)\le \Pr\left(\frac{r}{m_3}<\frac{\epsilon}{2}\right)\le \Pr\left(\frac{r}{m_3}<p-\frac{\epsilon}{2}\right)\le e^{-m_3\epsilon^2/2}\le\delta/4.
\end{equation*}
So the probability that this case happens is at most $\delta/4$.

\noindent{\textbf{3.} $r\ge m_2$\textbf{.}}

We note that
\begin{equation}
\label{equ: error of h}
\err(h)=\Pr(-H')\Pr(H_u\cap H_v|-H')+\Pr(H')\Pr(h_{xor}(x)\not=c(x)|x\in H'),
\end{equation}
where $c:\R^n\rightarrow \{-1,1\}$ is the hypothesis w.r.t. $H_u\cap H_v$. Observe that
\begin{equation*}
\Pr(-H')\Pr(H_u\cap H_v|-H')=\Pr(H_u\cap H_v)\Pr(-H'|H_u\cap H_v),
\end{equation*}
where $\Pr(-H'|H_u\cap H_v)$ is the error of $H'$ over the distribution conditioned on $H_u\cap H_v$. Since the VC argument works for any distribution, and $H'$ contains all $r\ge m_2$ positive examples according to Step 5 in Algorithm \ref{algorithm: intersections of halfspaces}, by the VC argument, with probability at least $1-\delta/4$,
\begin{equation*}
\Pr(-H'|H_u\cap H_v)\le \max\left\{\frac{\delta}{4(1+\gamma)^{1/\gamma} Km_1},\frac{\epsilon}{2}\right\}.
\end{equation*}
So
$
\Pr(-H')\Pr(H_u\cap H_v|-H')=\Pr(H_u\cap H_v\cap (-H'))\le \Pr(-H'|H_u\cap H_v)\le \max\left\{\frac{\delta}{4(1+\gamma)^{1/\gamma}Km_1},\frac{\epsilon}{2}\right\}\le \frac{\epsilon}{2}.
$

We now bound the second term in \eqref{equ: error of h}. According to Lemma \ref{lemma: isotropic reflection},
\begin{equation*}
\Pr((-H_u)\cap(-H_v)\cap H')\le K\Pr(H_u\cap H_v\cap (-H'))\le \frac{\delta}{4(1+\gamma)^{1/\gamma}m_1}.
\end{equation*}
On the other hand, by Lemma \ref{lemma: probability of halfspace}, $\Pr(H')\ge (1+\gamma)^{-1/\gamma}$ with $\gamma=s/(1+ns)$. Thus
\begin{equation*}
\Pr((-H_u)\cap (-H_v)|H')=\frac{\Pr((-H_u)\cap(-H_v)\cap(H'))}{\Pr(H')}\le \frac{\delta}{4m_1}.
\end{equation*}
That is to say, each point in $S$ has probability at most $\delta/(4m_1)$ of being in $(-H_u)\cap(-H_v)$. So by the union bound, with probability at least $1-\delta/4$, none of points in $S$ is in $(-H_u)\cap(-H_v)$. Therefore, Step 6 in Algorithm \ref{algorithm: intersections of halfspaces} is able to find $h_{xor}$ that is consistent with all the instances in $S$. Then by the VC argument, we have
\begin{equation*}
\Pr(h_{xor(x)}\not= c(x)|x\in H')\le \frac{\epsilon}{2},
\end{equation*}
with probability at least $1-\delta/4$. In summary, we have
\begin{equation*}
\begin{split}
\err(h)&=\Pr(-H')\Pr(H_u\cap H_v|-H')+\Pr(H')\Pr(h_{xor}(x)\not=c(x)|x\in H')\\
&\le \frac{\epsilon}{2}+\frac{\epsilon}{2}=\epsilon,
\end{split}
\end{equation*}
with failure probability at most $\delta/4+\delta/4+\delta=\delta$ by the union bound. Therefore, the proof is completed.
\end{proof}

\section{Proof of Lower Bounds}
\label{section: proof of lower bounds}
The proof of our lower bounds essentially depends on a lower bound on the packing number of all homogeneous linear separators $\mathbb{C}$ under distribution $\cD$. Remind that the $\epsilon$-packing number, denoted by $M_\cD(\mathbb{C},\epsilon)$, is the maximal cardinality of an $\epsilon$-separated set with classifiers from $\mathbb{C}$, where we say $N$ classifiers $w_1,...,w_N$ are $\epsilon$-separated w.r.t. $\cD$ if $d_\cD(w_i,w_j)\triangleq\Pr_{x\sim\cD}[\sign(w_i\cdot x)\not=\sign(w_j\cdot x)]>\epsilon$ for any $i\not=j$.

\begin{lemma}
\label{lemma: packing number}
Suppose that $\cD$ is $s$-concave in $\R^n$, and that its covariance matrix is of full rank. Then for all sufficiently small $\epsilon$, we have $M_\cD(\mathbb{C},\epsilon)\ge \frac{\sqrt{n}}{2}\left(\frac{f_1(s,n)}{2\epsilon}\right)^{n-1}-1$.
\end{lemma}

\begin{proof}
We begin with proving the lemma in the case of isotropic $\cD$. Our proof inspires from proofs for the special case of uniform and log-concave distributions by \cite{long1995sample} and \cite{balcan2013active}, respectively.

Denote by $\textsf{UBALL}_n$ the uniform distribution on the sphere in $\R^n$. According to Theorem \ref{theorem: disagreement and angle}, for any two unit vectors $u$ and $v$ in $\R^n$ we have $f_1(s,n)\theta(u,v)\le d_\cD(u,v)$. Thus for a fixed $u$ the probability that a uniformly chosen $v$ obeys $d_\cD(u,v)\le\epsilon$ is upper bounded by the volume of those points in the interior of unit ball whose angle is at most $\epsilon/f_1(s,n)$ divided by the volume of unit ball in $\R^n$. By known bound on this ratio~\cite{long1995sample}, we have $\Pr_{v\in\textsf{UBALL}_n}[d_\cD(u,v)\le\epsilon]\le \frac{1}{\sqrt{n}}\left(\frac{2\epsilon}{f_1(s,n)}\right)^{n-1}$. So $\Pr_{u,v\in\textsf{UBALL}_n}[d_\cD(u,v)\le\epsilon]\le \frac{1}{\sqrt{n}}\left(\frac{2\epsilon}{f_1(s,n)}\right)^{n-1}$, meaning that if we select a set $S$ of $s$ normalized vectors uniformly from the unit sphere, the expected number of pairs of vectors that are $\epsilon$-close in the sense of $d_\cD$ is at most $\frac{s^2}{\sqrt{n}}\left(\frac{2\epsilon}{f_1(s,n)}\right)^{n-1}$. Removing one vector from each pair of $S$ yields a set of $s-\frac{s^2}{\sqrt{n}}\left(\frac{2\epsilon}{f_1(s,n)}\right)^{n-1}$ homogeneous linear separators that are $\epsilon$-separated. The proof for isotropic $\cD$ is completed when we set $s=\frac{\sqrt{n}}{(2\epsilon/f_1(s,n))^{n-1}}$.

We now discuss the case when $\cD$ is non-isotropic. Denote by $\Sigma$ the covariance matrix of $\cD$ and let isotropic $\cD'$ be the whitened version of $\cD$, namely, the distribution obtained by first sampling $x$ from $\cD$ and then computing $\Sigma^{-1/2}x$. Notice that $d_\cD(u,v)=d_{\cD'}(u\Sigma^{1/2},v\Sigma^{1/2})$. Therefore, we can apply an $\epsilon$-packing w.r.t. $\cD'$ to construct an $\epsilon$-packing w.r.t. $\cD'$ of the same size.
\end{proof}

Now we are ready to prove Theorem \ref{theorem: lower bounds of learning halfspace}.
\medskip

\noindent{\textbf{Theorem \ref{theorem: lower bounds of learning halfspace}} (restated)}
\emph{For a fixed value $-\frac{1}{2n+3}\le s\le 0$ we have: (a) For any $s$-concave distribution $\cD$ in $\R^n$ whose covariance matrix is of full rank, the sample complexity of learning origin-centered linear separators under $\cD$ in the passive learning scenario is $\Omega\left(\frac{nf_1(s,n)}{\epsilon}\right)$; (b) The label complexity of active learning of linear separators under $s$-concave distribution is $\Omega\left(n\log\left(\frac{f_1(s,n)}{\epsilon}\right)\right)$.}

\medskip
\begin{proof}
It is known that for any distribution $\cD$ in $\R^n$, the sample complexity of (passive) PAC learning of homogeneous linear separators under $\cD$ is at least $\frac{n-1}{e}\left(\frac{M_\cD(\mathbb{C},2\epsilon)}{4}\right)^{1/(n-1)}$~\cite{long1995sample}. By Lemma \ref{lemma: packing number}, we have an $\Omega\left(\frac{nf_1(s,n)}{\epsilon}\right)$ lower bound of sample complexity for passive learning homogeneous halfspace.

We now discuss the label complexity lower bound in the active learning scenario. By \cite{kulkarni1993active}, any active learning algorithm that is allowed to make arbitrary binary queries must take at least $\Omega(\log M_{\cD}(\mathbb{C},\epsilon))$ so as to output a hypothesis of error at most $\epsilon$ with high probability. Applying Lemma \ref{lemma: packing number}, we obtain the desired result.
\end{proof}

\section{Related Algorithms}
\label{section: algorithms}

\subsection{Margin Based Active Learning (Realizable Case)}

\begin{algorithm}[ht]
\caption{Margin Based Active Learning under S-Concave Distributions (Realizable Case)}
\begin{algorithmic}
\label{algorithm: margin based active learning realizable case}
\STATE {\bfseries Input:} $b_k=\min\{\Theta(2^{-k}f_4f_1^{-1}),d\}$, $m_k=C\left(\frac{f_3b_{k-1}}{2^{-k}}\left(n\log \frac{f_3b_{k-1}}{2^{-k}}+\log \frac{1+s-k}{\delta}\right)\right)$, and $T=\lceil\log \frac{1}{c\epsilon}\rceil$.
\STATE {\bfseries 1:} Draw $m_1$ examples from $\cD$, label them and put them into $W(1)$.
\STATE {\bfseries 2:} \textbf{For} $k=1,2,...,T$
\STATE {\bfseries 3:} \quad Find a hypothesis $w_k$ with $\|w_k\|=1$ that is consistent with $W(k)$.
\STATE {\bfseries 4:} \quad $W(k+1)\leftarrow W(k)$.
\STATE {\bfseries 5:} \quad \textbf{While} $m_{k+1}$ additional data points are not labeled
\STATE {\bfseries 6:} \quad \quad Draw sample $x$ from $\cD$.
\STATE {\bfseries 7:} \quad \quad \textbf{If} $|w_k\cdot x|\ge b_k$
\STATE {\bfseries 8:} \quad \quad \quad Reject $x$.
\STATE {\bfseries 9:} \quad \quad \textbf{Else}
\STATE {\bfseries 10:} \quad \quad \quad Ask for label of $x$ and put into $W(k+1)$.
\STATE {\bfseries 11:} \quad \ \ \textbf{End If}
\STATE {\bfseries 12:} \ \ \textbf{End While}
\STATE {\bfseries 13:} \textbf{End For}
\STATE {\bfseries Output:} Hypothesis $w_T$.
\end{algorithmic}
\end{algorithm}

\subsection{Margin Based Active Learning (Adversarial Noise)}

\begin{algorithm}[h]
\floatname{algorithm}{Procedure}
\caption{Margin Based Active Learning under S-Concave Distributions (Adversarial Noise)}
\begin{algorithmic}
\STATE {\bfseries Input:} Parameters $b_k$, $\tau_k$, $r_k$, $m_k$, $\kappa$, and $T$ as in Theorem \ref{theorem: margin-based active learning adversarial noise}.
\STATE {\bfseries 1:} Draw $m_1$ examples from $\cD$, label them and put them into $W$.
\STATE {\bfseries 2:} \ \ \textbf{For} $k=1,2,...,T$
\STATE {\bfseries 3:} \quad Find $v_k\in \ball(w_{k-1},r_k)$ to approximately minimize the hinge loss over $W$ s.t. $\|v_k\|\le 1$:\\
\qquad\ $\ell_{\tau_k}\le \min_{w\in \ball(w_{k-1},r_k)\cap \ball(0,1)} \ell_{\tau_k}(w,W)+\kappa/8$.
\STATE {\bfseries 4:} \quad Normalize $v_k$, yielding $w_k=\frac{v_k}{\|v_k\|}$.
\STATE {\bfseries 5:} \quad Clear the working set $W$.
\STATE {\bfseries 6:} \quad \textbf{While} $m_{k+1}$ additional data points are not labeled
\STATE {\bfseries 7:} \quad \quad Draw sample $x$ from $\cD$.
\STATE {\bfseries 8:} \quad \quad \textbf{If} $|w_k\cdot x|\ge b_k$, reject $x$; \textbf{else} ask for label of $x$ and put into $W$.
\STATE {\bfseries 9:} \ \ \ \ \textbf{End While}
\STATE {\bfseries 10:} \textbf{End For}
\STATE {\bfseries Output:} Hypothesis $w_T$.
\end{algorithmic}
\end{algorithm}

\subsection{Learning Intersections of Halfspaces}

\begin{algorithm}[h]
\caption{Learning Intersections of Halfspaces under S-Concave Distributions}
\begin{algorithmic}
\label{algorithm: intersections of halfspaces}
\STATE {\bfseries Input:} Parameters $m_1$, $m_2$, and $m_3$ as in Theorem \ref{theorem: Baum's algorithm}.
\STATE {\bfseries 1:} Draw $m_3$ examples. Denote by $r$ the number of observed positive examples.
\STATE {\bfseries 2:} \textbf{If} $r<m_2$, output the hypothesis that labels every point as negative, and end the algorithm.
\STATE {\bfseries 3:} Learn an origin-centered halfspace $H'$ which contains all $r$ positive examples.
\STATE {\bfseries 4:} Draw a set $S$ of $m_1$ i.i.d. examples in $H'$. Learn a weight vector $w\in\R^{n\times n}$ such that the hypothesis
$h_{xor}=\sign\left(\sum_{i=1}^n\sum_{j=1}^n w_{ij}x_ix_j\right)$
is consistent with the set $S$.
\STATE {\bfseries Output:} $h:\R^n\rightarrow\{-1,1\}$ such that $h(x)=h_{xor}(x)$ if $x\in H'$; Otherwise, $h(x)=-1$.
\end{algorithmic}
\end{algorithm}

\section{A Collection of Concentration Results}

\begin{theorem}[\cite{vapnik1982estimations,blumer1989learnability}]
\label{theorem: vc theory}
Denote by $\cC$ a class of concepts from a set $X$ to $\{-1,1\}$ with VC dimension $n$. Let $c\in\cC$, and assume that
\begin{equation*}
M(\epsilon,\delta,n)=O\left(\frac{n}{\epsilon}\log \frac{1}{\epsilon}+\frac{1}{\epsilon}\log\frac{1}{\delta}\right)
\end{equation*}
examples $x_1,...,x_M$ are sampled from any probability distribution $\cD$ over $X$. Then any hypothesis $h\in\cC$ which is consistent with $c$ on $x_1,...,x_M$ has error at most $\epsilon$, with probability at least $1-\delta$.
\end{theorem}

\begin{theorem}[\cite{anthony2009neural}]
\label{theorem: pseudo-dimension argument}
Let $F$ be a set of functions mapping from domain $X$ to $[a,b]$, and let $n$ be the pseudo-dimension of $F$. Then for any distribution $\cD$ over $X$ and $m=O\left(\frac{(b-a)^2}{\kappa^2}(d+\log (1/\delta))\right)$, if $x_1,...,x_m$ are drawn independently from $\cD$, with probability at least $1-\delta$, for all $f\in F$,
\begin{equation*}
\left|\E_{x\sim \cD}f(x)-\frac{1}{m}\sum_{i=1}^mf(x_i)\right|\le \kappa.
\end{equation*}
\end{theorem}

\end{document}